\documentclass[a4paper,conference]{IEEEtran}
\ifCLASSINFOpdf
\else
\fi
\usepackage{graphicx}
\usepackage{amsmath,amssymb,amsthm}
\usepackage{color}
\usepackage{times}
\usepackage{epsfig}
\usepackage{enumerate}
\usepackage{algorithm}
\usepackage{algorithmic}
\usepackage{comment}
\usepackage{url}
\usepackage{booktabs}
\usepackage{multirow}
\usepackage{tikz,tikz-3dplot}
\usepackage{pgfplots}
\usepgfplotslibrary{units}
\usepackage{amsfonts}
\usepackage{xcolor}
\usepackage{float}
\usepackage{arydshln}
 
\def \ie{$i.e.$}
\def \etal{$et~al.$}
\def \eg{$e.g.$}
\def \calA{\mathcal{A}}
\def \calB{\mathcal{B}}
\def \calC{\mathcal{C}}
\def \calD{\mathcal{D}}
\def \calE{\mathcal{E}}
\def  \calI{\mathcal{I}}
\def  \calP{\mathcal{P}}
\def  \calQ{\mathcal{Q}}
\def  \calU{\mathcal{U}}
\def  \calR{\mathcal{R}}
\def  \calS{\mathcal{S}}
\def  \calV{\mathcal{V}}
\def  \calT{\mathcal{T}}
\def  \calX{\mathcal{X}}
\def  \calY{\mathcal{Y}}
\def  \calW{\mathcal{W}}
\def  \calM{\mathcal{M}}
\def  \calL{\mathcal{L}}
\def  \calO{\mathcal{O}}
\def \calZ{\mathcal{Z}}

\def \bbC{\mathbb{C}}
\def \bbR{\mathbb{R}}

\def  \fft{\texttt{fft}}
\def  \ifft{\texttt{ifft}}

\def \old{\textnormal{old}}

\def\NC {\varphi_{\lambda,\gamma}(t)}
\def\NCg {\varphi_{1,\gamma}(t)}
\def\SCAD {\varphi_{\lambda,\gamma}^{\mathrm{SCAD}}(t)}
\def\SCADone {\varphi_{\lambda,\gamma_1}^{\mathrm{SCAD}}(t)}
\def\SCADtwo {\varphi_{\lambda,\gamma_2}^{\mathrm{SCAD}}(t)}
\def\MCP {\varphi_{\lambda,\gamma}^{\mathrm{MCP}}(t)}
\def\MCPone {\varphi_{\lambda,\gamma_1}^{\mathrm{MCP}}(t)}
\def\MCPtwo {\varphi_{\lambda,\gamma_2}^{\mathrm{MCP}}(t)}

\newtheorem{thm}{Theorem}[section]
\newtheorem{lem}[thm]{Lemma}
\newtheorem{prop}[thm]{Proposition}
\newtheorem{defn}{Definition}[section]

\begin{document}
%
\title{T-SVD Based Non-convex Tensor Completion and Robust Principal Component Analysis}

\author{\IEEEauthorblockN{Tao Li}
\IEEEauthorblockA{Department of Information Science\\School of Mathematical Sciences and LMAM\\
Peking University, Beijing, China\\
}
\and
\IEEEauthorblockN{Jinwen Ma}
\IEEEauthorblockA{Department of Information Science\\ School of Mathematical Sciences and LMAM\\
Peking University, Beijing, China\\
Email: {\tt\small jwma@math.pku.edu.cn}
}
}


%


\maketitle

\begin{abstract}
Tensor completion and robust principal component analysis have been widely used in machine learning while the key problem relies on the minimization of a tensor rank that is very challenging. A common way to tackle this difficulty is to approximate the tensor rank with the $\ell_1-$norm of singular values based on its Tensor Singular Value Decomposition (T-SVD). Besides, the sparsity of a tensor is also measured by its $\ell_1-$norm. However, the $\ell_1$ penalty is essentially biased and thus the result will deviate. In order to sidestep the bias, we propose a novel non-convex tensor rank surrogate function and a novel non-convex sparsity measure. In this new setting by using the concavity instead of the convexity, a majorization minimization algorithm is further designed for tensor completion and robust principal component analysis. Furthermore, we analyze its theoretical properties.  Finally, the experiments on natural and hyperspectral images demonstrate the efficacy and efficiency of our proposed method. 
\end{abstract}


%
\IEEEpeerreviewmaketitle

\section{Introduction}
Tensors \cite{kolda}, or multi-way arrays,  have been extensively used in computer vision \cite{SNN,BCPF}, signal processing and machine learning \cite{cichocki2015tensor,sidiropoulos2017tensor}. 
Due to technical reasons, tensors in most applications are incomplete or polluted. Generically, recovering a tensor from corrupted observations is an inverse problem, which is ill-posed without prior knowledge. However, in real applications, entries in a tensor are usually highly correlated, which means a high-dimensional tensor is intrinsically determined by low-dimensional factors. Exploiting such low-dimensional structure makes it possible to restore tensors from limited or corrupted observations. Mathematically, this prior knowledge is equivalent to assume the tensors are low-rank.

In this work, we mainly consider two tensor recovery problems: tensor completion and tensor robust principal component analysis (TRPCA). The tensor completion problem is to estimate the missing values in tensors from partially observed data, while TRPCA aims to decompose a tensor into a low-rank tensor and sparse tensor.
In the case of 2-order tensor, \ie, the matrix case, both problems have been investigated thoroughly \cite{candes2009,cai2010singular,candes2011robust,wright2009robust}. Since the concept of a tensor is an extension of the matrix, it is natural to employ matrix recovery methods to tensors. Most matrix recovery methods are optimization-based, penalizing rank surrogate function, or/and certain sparsity measure. Similar methods have been developed for tensors Tensor $\ell_1-$norm is often used as sparsity measure. However, the concept of tensor rank is far more complicated than matrix rank, thus there are various surrogate functions for tensor rank, such as the sum of the nuclear norm (SNN) \cite{SNN}, tensor nuclear norm (TNN) \cite{TNN} and twisted tensor nuclear norm (t-TNN) \cite{tTNN}. 

\begin{figure}[htb]
\vspace{-1em}
\centering
\begin{tabular}{c}
\newcommand{\nh}{1.5}
\newcommand{\nw}{1.5}
\newcommand{\nd}{1.5}
\begin{tikzpicture}[scale=0.8]
\begin{scope}[shift={(-0.9,0)}]
\coordinate (O) at (0,0,0);
\coordinate (A) at (0,\nh,0);
\coordinate (D) at (\nw,0,0);
\coordinate (E) at (\nw,\nh,0);
\coordinate (C) at (0,0,\nd);
\coordinate (B) at (0,\nh,\nd);
\coordinate (G) at (\nw,0,\nd);
\coordinate (F) at (\nw,\nh,\nd);

\draw[red!60!black,fill=red!10] (D) -- (E) -- (F) -- (G) -- (D);
\draw[red!60!black,fill=red!10] (C) -- (B) -- (F) -- (G) -- (C);
\draw[red!60!black,fill=red!10] (A) -- (B) -- (F) -- (E) --(A);
\draw[red!60!black,dashed] (A)--(O) --(D);
\draw[red!60!black,dashed] (O) --(C);
\draw (0.5,-1,0) node {\begin{small}Incomplete tensor\end{small}};

\coordinate (O) at (0.7,0.4,0);
\coordinate (A) at (1,0.4,0);
\coordinate (D) at (0.7,0.6,0);
\coordinate (E) at (1,0.6,0);
\coordinate (C) at (0.7,0.4,0.5*\nd);
\coordinate (B) at (1,0.4,0.5*\nd);
\coordinate (G) at (0.7,0.6,0.5*\nd);
\coordinate (F) at (1,0.6,0.5*\nd);
\draw[red!20!white,fill=white] (D) -- (E) -- (F) -- (G) -- cycle;
\draw[red!20!white,fill=white] (C) -- (B) -- (F) -- (G) -- cycle;
\draw[red!20!white,fill=white] (A) -- (B) -- (F) -- (E) -- cycle;

\coordinate (O) at (0.1,1,0.2);
\coordinate (A) at (0.3,1,0.2);
\coordinate (D) at (0.1,1.2,0.2);
\coordinate (E) at (0.3,1.2,0.2);
\coordinate (C) at (0.1,1,0.5*\nd);
\coordinate (B) at (0.3,1,0.5*\nd);
\coordinate (G) at (0.1,1.2,0.5*\nd);
\coordinate (F) at (0.3,1.2,0.5*\nd);
\draw[red!20!white,fill=white] (D) -- (E) -- (F) -- (G) -- cycle;
\draw[red!20!white,fill=white] (C) -- (B) -- (F) -- (G) -- cycle;
\draw[red!20!white,fill=white] (A) -- (B) -- (F) -- (E) -- cycle;

\coordinate (O) at (0.2,0.2,0.1);
\coordinate (A) at (0.4,0.2,0.1);
\coordinate (D) at (0.2,0.4,0.1);
\coordinate (E) at (0.4,0.4,0.1);
\coordinate (C) at (0.2,0.2,0.7*\nd);
\coordinate (B) at (0.4,0.2,0.7*\nd);
\coordinate (G) at (0.2,0.4,0.7*\nd);
\coordinate (F) at (0.4,0.4,0.7*\nd);
\draw[red!20!white,fill=white] (D) -- (E) -- (F) -- (G) -- cycle;
\draw[red!20!white,fill=white] (C) -- (B) -- (F) -- (G) -- cycle;
\draw[red!20!white,fill=white] (A) -- (B) -- (F) -- (E) -- cycle;

\coordinate (O) at (1.,1,0.1);
\coordinate (A) at (1.2,1.,0.1);
\coordinate (D) at (1.,1.2,0.1);
\coordinate (E) at (1.2,1.2,0.1);
\coordinate (C) at (1.,1.,0.7*\nd);
\coordinate (B) at (1.2,1.,0.7*\nd);
\coordinate (G) at (1.,1.2,0.7*\nd);
\coordinate (F) at (1.2,1.2,0.7*\nd);
\draw[red!20!white,fill=white] (D) -- (E) -- (F) -- (G) -- cycle;
\draw[red!20!white,fill=white] (C) -- (B) -- (F) -- (G) -- cycle;
\draw[red!20!white,fill=white] (A) -- (B) -- (F) -- (E) -- cycle;

\coordinate (B) at (0,\nh,\nd);
\coordinate (F) at (\nw,\nh,\nd);
\coordinate (G) at (\nw,0,\nd);
\coordinate (E) at (\nw,\nh,0);
\draw[red!60!black] (F) -- (G);
\draw[red!60!black] (B) -- (F) ;
\draw[red!60!black] (E) -- (F) ;
\end{scope}

\begin{scope}[line width=5pt]
\draw[gray!20,->] (0.8,0.65) -- (1.4,0.65);
\end{scope}

\begin{scope}[shift={(2.1,0)}]
\coordinate (O) at (0,0,0);
\coordinate (A) at (0,\nh,0);
\coordinate (D) at (\nw,0,0);
\coordinate (E) at (\nw,\nh,0);
\coordinate (C) at (0,0,\nd);
\coordinate (B) at (0,\nh,\nd);
\coordinate (G) at (\nw,0,\nd);
\coordinate (F) at (\nw,\nh,\nd);

\draw[red!60!black,fill=red!10] (D) -- (E) -- (F) -- (G) -- (D);
\draw[red!60!black,fill=red!10] (C) -- (B) -- (F) -- (G) -- (C);
\draw[red!60!black,fill=red!10] (A) -- (B) -- (F) -- (E) --(A);
\draw[red!60!black,dashed] (A)--(O) --(D);
\draw[red!60!black,dashed] (O) --(C);
\draw (0.56,-1,0) node {\begin{small}Complete tensor\end{small}};
\end{scope}
\end{tikzpicture}\\
\newcommand{\nh}{1.5}
\newcommand{\nw}{1.5}
\newcommand{\nd}{1.5}

\begin{tikzpicture}[scale=0.8]
\begin{scope}[shift={(-1.2,0)}]
\coordinate (O) at (0,0,0);
\coordinate (A) at (0,\nh,0);
\coordinate (D) at (0.3*\nw,0,0);
\coordinate (E) at (0.3*\nw,\nh,0);
\coordinate (C) at (0,0,\nd);
\coordinate (B) at (0,\nh,\nd);
\coordinate (G) at (0.3*\nw,0,\nd);
\coordinate (F) at (0.3*\nw,\nh,\nd);

\draw[red!60!black,fill=orange!10] (D) -- (E) -- (F) -- (G) -- (D);
\draw[red!60!black,fill=orange!10] (C) -- (B) -- (F) -- (G) -- (C);
\draw[red!60!black,fill=orange!10] (A) -- (B) -- (F) -- (E) --(A);

\coordinate (O) at (0.3*\nw0,0,0);
\coordinate (A) at (0.3*\nw,\nh,0);
\coordinate (D) at (0.7*\nw,0,0);
\coordinate (E) at (0.7*\nw,\nh,0);
\coordinate (C) at (0.3*\nw,0,\nd);
\coordinate (B) at (0.3*\nw,\nh,\nd);
\coordinate (G) at (0.7*\nw,0,\nd);
\coordinate (F) at (0.7*\nw,\nh,\nd);

\draw[red!60!black,fill=red!10] (D) -- (E) -- (F) -- (G) -- (D);
\draw[red!60!black,fill=red!10] (C) -- (B) -- (F) -- (G) -- (C);
\draw[red!60!black,fill=red!10] (A) -- (B) -- (F) -- (E) --(A);

\coordinate (O) at (0.7*\nw0,0,0);
\coordinate (A) at (0.7*\nw,\nh,0);
\coordinate (D) at (\nw,0,0);
\coordinate (E) at (\nw,\nh,0);
\coordinate (C) at (0.7*\nw,0,\nd);
\coordinate (B) at (0.7*\nw,\nh,\nd);
\coordinate (G) at (\nw,0,\nd);
\coordinate (F) at (\nw,\nh,\nd);

\draw[red!60!black,fill=blue!10] (D) -- (E) -- (F) -- (G) -- (D);
\draw[red!60!black,fill=blue!10] (C) -- (B) -- (F) -- (G) -- (C);
\draw[red!60!black,fill=blue!10] (A) -- (B) -- (F) -- (E) --(A);

\draw (0.56,-1,0) node {\begin{small}Corrupted tensor\end{small}};

\draw plot [mark=*, mark size=1,draw=blue!20] coordinates{(0,0,1)};
\draw plot [mark=*, mark size=1,draw=blue!20] coordinates{(0.2,0.5,1)};
\draw plot [mark=*, mark size=1,draw=blue!20] coordinates{(0.7,0.3,1)};
\draw plot [mark=*, mark size=1,draw=blue!20] coordinates{(1.1,0.2,1)};
\draw plot [mark=*, mark size=1,draw=blue!20] coordinates{(0.8,0.9,1.2)};
\draw plot [mark=*, mark size=1,draw=blue!20] coordinates{(0.27,0.15,0.61)};
\draw plot [mark=*, mark size=1,draw=blue!20] coordinates{(0.9,1.2,1.1)};
\draw plot [mark=*, mark size=1,draw=blue!20] coordinates{(1.2,0.9,1.3)};
\draw plot [mark=*, mark size=1,draw=blue!20] coordinates{(0.9,0.1,0.1)};
\draw plot [mark=*, mark size=1,draw=blue!20] coordinates{(0.8,0.1,0.2)};
\draw plot [mark=*, mark size=1,draw=blue!20] coordinates{(0.2,0.9,1.1)};
\draw plot [mark=*, mark size=1,draw=blue!20] coordinates{(0.3,0.8,1.5)};
\draw plot [mark=*, mark size=1,draw=blue!20] coordinates{(0.1,0.3,1.4)};
\draw plot [mark=*, mark size=1,draw=blue!20] coordinates{(0.5,1.6,1.1)};
\draw plot [mark=*, mark size=1,draw=blue!20] coordinates{(0.2,1.5,1.1)};
\draw plot [mark=*, mark size=1,draw=blue!20] coordinates{(0.9,1.4,0.8)};
\draw plot [mark=*, mark size=1,draw=blue!20] coordinates{(1.12,1.3,0.2)};
\draw plot [mark=*, mark size=1,draw=blue!20] coordinates{(1.5,0.5,1.3)};
\draw plot [mark=*, mark size=1,draw=blue!20] coordinates{(1.6,0.7,0.6)};
\draw plot [mark=*, mark size=1,draw=blue!20] coordinates{(1.4,1.2,0.9)};
\end{scope}

\begin{scope}[line width=5pt]
\draw[gray!20,->] (0.5,0.65) -- (1.2,0.65);
\end{scope}

\begin{scope}[shift={(1.9,0)}]
\coordinate (O) at (0,0,0);
\coordinate (A) at (0,\nh,0);
\coordinate (D) at (0.3*\nw,0,0);
\coordinate (E) at (0.3*\nw,\nh,0);
\coordinate (C) at (0,0,\nd);
\coordinate (B) at (0,\nh,\nd);
\coordinate (G) at (0.3*\nw,0,\nd);
\coordinate (F) at (0.3*\nw,\nh,\nd);

\draw[red!60!black,fill=orange!10] (D) -- (E) -- (F) -- (G) -- (D);
\draw[red!60!black,fill=orange!10] (C) -- (B) -- (F) -- (G) -- (C);
\draw[red!60!black,fill=orange!10] (A) -- (B) -- (F) -- (E) --(A);

\coordinate (O) at (0.3*\nw0,0,0);
\coordinate (A) at (0.3*\nw,\nh,0);
\coordinate (D) at (0.7*\nw,0,0);
\coordinate (E) at (0.7*\nw,\nh,0);
\coordinate (C) at (0.3*\nw,0,\nd);
\coordinate (B) at (0.3*\nw,\nh,\nd);
\coordinate (G) at (0.7*\nw,0,\nd);
\coordinate (F) at (0.7*\nw,\nh,\nd);

\draw[red!60!black,fill=red!10] (D) -- (E) -- (F) -- (G) -- (D);
\draw[red!60!black,fill=red!10] (C) -- (B) -- (F) -- (G) -- (C);
\draw[red!60!black,fill=red!10] (A) -- (B) -- (F) -- (E) --(A);

\coordinate (O) at (0.7*\nw0,0,0);
\coordinate (A) at (0.7*\nw,\nh,0);
\coordinate (D) at (\nw,0,0);
\coordinate (E) at (\nw,\nh,0);
\coordinate (C) at (0.7*\nw,0,\nd);
\coordinate (B) at (0.7*\nw,\nh,\nd);
\coordinate (G) at (\nw,0,\nd);
\coordinate (F) at (\nw,\nh,\nd);

\draw[red!60!black,fill=blue!10] (D) -- (E) -- (F) -- (G) -- (D);
\draw[red!60!black,fill=blue!10] (C) -- (B) -- (F) -- (G) -- (C);
\draw[red!60!black,fill=blue!10] (A) -- (B) -- (F) -- (E) --(A);

\draw (0.56,-1,0) node {\begin{small}Low-rank tensor\end{small}};
\end{scope}

\draw(3.8,0.5) node{+};

\begin{scope}[shift={(4.6,0)}]
\coordinate (O) at (0,0,0);
\coordinate (A) at (0,\nh,0);
\coordinate (D) at (\nw,0,0);
\coordinate (E) at (\nw,\nh,0);
\coordinate (C) at (0,0,\nd);
\coordinate (B) at (0,\nh,\nd);
\coordinate (G) at (\nw,0,\nd);
\coordinate (F) at (\nw,\nh,\nd);

\draw[red!60!black,fill=red!0] (D) -- (E) -- (F) -- (G) -- (D);
\draw[red!60!black,fill=red!0] (C) -- (B) -- (F) -- (G) -- (C);
\draw[red!60!black,fill=red!0] (A) -- (B) -- (F) -- (E) --(A);
\draw[red!60!black,dashed] (A)--(O) --(D);
\draw[red!60!black,dashed] (O) --(C);

\draw plot [mark=*, mark size=1,draw=blue!20] coordinates{(0,0,1)};
\draw plot [mark=*, mark size=1,draw=blue!20] coordinates{(0.2,0.5,1)};
\draw plot [mark=*, mark size=1,draw=blue!20] coordinates{(0.7,0.3,1)};
\draw plot [mark=*, mark size=1,draw=blue!20] coordinates{(1.1,0.2,1)};
\draw plot [mark=*, mark size=1,draw=blue!20] coordinates{(0.8,0.9,1.2)};
\draw plot [mark=*, mark size=1,draw=blue!20] coordinates{(0.27,0.15,0.61)};
\draw plot [mark=*, mark size=1,draw=blue!20] coordinates{(0.9,1.2,1.1)};
\draw plot [mark=*, mark size=1,draw=blue!20] coordinates{(1.2,0.9,1.3)};
\draw plot [mark=*, mark size=1,draw=blue!20] coordinates{(0.9,0.1,0.1)};
\draw plot [mark=*, mark size=1,draw=blue!20] coordinates{(0.8,0.1,0.2)};
\draw plot [mark=*, mark size=1,draw=blue!20] coordinates{(0.2,0.9,1.1)};
\draw plot [mark=*, mark size=1,draw=blue!20] coordinates{(0.3,0.8,1.5)};
\draw plot [mark=*, mark size=1,draw=blue!20] coordinates{(0.1,0.3,1.4)};
\draw plot [mark=*, mark size=1,draw=blue!20] coordinates{(0.5,1.6,1.1)};
\draw plot [mark=*, mark size=1,draw=blue!20] coordinates{(0.2,1.5,1.1)};
\draw plot [mark=*, mark size=1,draw=blue!20] coordinates{(0.9,1.4,0.8)};
\draw plot [mark=*, mark size=1,draw=blue!20] coordinates{(1.12,1.3,0.2)};
\draw plot [mark=*, mark size=1,draw=blue!20] coordinates{(1.5,0.5,1.3)};
\draw plot [mark=*, mark size=1,draw=blue!20] coordinates{(1.6,0.7,0.6)};
\draw plot [mark=*, mark size=1,draw=blue!20] coordinates{(1.4,1.2,0.9)};
\draw (0.5,-1,0) node {\begin{small}Sparse tensor\end{small}};

\end{scope}
\end{tikzpicture}\\
\end{tabular}
\caption{An illustration of tensor completion and TRPCA.}
\label{fig:tcandtrpca}
\vspace{-1em}
\end{figure}
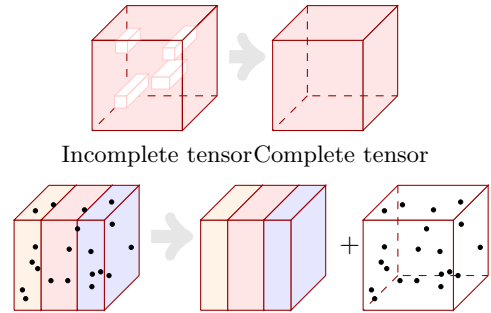

As in the matrix case, the choices of rank surrogate function and sparsity measure substantially influence the final results. The nuclear norm of a matrix is equivalent to the $\ell_1-$norm of its singular value. However, as indicated by Fan and Li \cite{SCAD}, $\ell_1-$norm over-penalizes large entries of vectors. 
 Smoothly clipped absolute deviation (SCAD) penalty \cite{SCAD} and minimax concave plus (MCP) penalty \cite{MCP} were proposed as ideal penalty functions, and their superiority over $\ell_1-$norm has been demonstrated in \cite{MCP,SCAD,shi2011,zhang2012}. This fact inspires us that nuclear norm based tensor rank surrogate functions and $\ell_1-$norm based tensor sparsity measure may suffer from a similar problem. To alleviate such phenomena, we propose to use non-convex penalties (SCAD and MCP) instead of an $\ell_1-$norm in TNN and tensor sparsity measures. 

However, the introduction of non-convex penalties makes optimization problems even harder to solve. For example, TNN based TRPCA \cite{TNN} is a convex optimization model, thus can be efficiently solved by alternating direction multiplier method (ADMM) \cite{ADMM}. Once we replace $\ell_1-$norm by SCAD or MCP, the problem is not convex anymore.
 Therefore, we apply the majorization minimization algorithm \cite{MM1,MM2} to solve the non-convex optimization problems and analyze the theoretical properties of these algorithms. Based on the proposed non-convex tensor completion and TRPCA models and their corresponding MM algorithms, we conduct experiments on natural images and multispectral images to validate the efficacy of the proposed methods.


\section{Related Work}\label{sec:relatedwork}
{\bf Tensor recovery.} For tensor completion, one seminal work is \cite{SNN}, in which SNN was proposed and three different algorithms for solving SNN based TC were devised. Zhao \etal proposed Bayesian CANDECOMP/PARAFAC (CP) tensor factorization model in \cite{BCPF}. The highlights of \cite{BCPF} include automatic rank determination property, full Bayesian treatment, and uncertainty quantification. Kilmer and Martin proposed a new tensor singular value decomposition (t-SVD) based on discrete Fourier transform for $3-$order tensors in \cite{kilmer2011,kilmer2013}. The key point is that t-SVD offers an efficient way to define tensor nuclear norm (TNN), which has been extensively used in tensor recovery recently \cite{ZeminZhang,TNN,zhou2017outlier,TNNPAMI}. Furthermore, Lu \etal proved the the exact recovery property of their proposed TRPCA model under certain suitable assumptions \cite{TNN,TNNPAMI}.

{\bf Non-convex penalties.} Wang and Zhang \cite{wang2013nonconvex} developed a non-convex optimization model for the low-rank matrix recovery problem. Cao \etal \cite{Folded} applied folded-concave penalties in SNN, while Ji \etal \cite{ji2017non} used log determinant penalty instead. Zhao \etal \cite{zhao2015novel} proposed to use the product of nuclear norm instead of the sum of the nuclear norm, which has a natural physical meaning. Besides, they also considered non-convex penalties such as SCAD and MCP. One major difference between our work and \cite{Folded,zhao2015novel} is that our methods are based on t-SVD, while \cite{Folded,zhao2015novel} transform a tensor to matrices simply via unfolding. Recently, Jiang \etal \cite{jiang2017novel} and Xu \etal \cite{xu2018laplace} introduced non-convex penalties to TNN, but neither MCP nor SCAD was considered. Besides, our work not only improves the tensor rank surrogate function but also modifies the tensor sparsity measure. Yokota \etal \cite{yokota2018missing} proposed a Tucker decomposition based non-convex tensor completion model for a case where all of the elements in some continuous slices are missing. Yao \etal \cite{yao2019efficient} also considered non-convex regularizers for tensor recovery and devised an efficient solver based on the proximal average algorithm. However, the non-convex penalties in \cite{yao2019efficient} are based on SNN, while ours are based on TNN.

\section{Notations and Preliminaries}\label{sec:notations}
\subsection{Notations}
Throughout this paper, we use calligraphic letters to denote 3-way tensors, \eg, $\calA\in\bbC^{n_1\times n_2\times n_3}$. The $(i,j,k)$-th element of $\calA$ may be denoted by $\calA(i,j,k)$ or $\calA_{ijk}$ alternatively. The $k$-th frontal slice of $\calA$ is defined as $\calA(:,:,k)$, which is an $n_1\times n_2$ matrix. For brevity, we use $A^{(k)}$ to denote $\calA(:,:,k)$. The $(i,j)$-th tube of $\calA$ is defined as $\calA(i,j,:)$, which is a vector of length $n_3$. The inner product of two 3-way tensors $\calA,\calB\in \bbC^{n_1\times n_2\times n_3}$ is defined as $\langle\calA,\calB\rangle =\sum_{k}\text{Tr}((A^{(k)})^*B^{(k)})$. We use $|\calA|$ to denote the tensor with $(i,j,k)-$th element equals to $|\calA_{ijk}|$. Similar to vectors and matrices, we can also define various norms of tensors. We denote $\ell_1-$norm by $\|\calA\|_1 = \sum_{ijk}|\calA_{ijk}|$, $\ell_{\infty}-$norm by $\|\calA\|_{\infty} = \max_{ijk}|\calA_{ijk}|$ and Frobenius norm by $\|\calA\|_F=\sqrt{\sum_{ijk}|\calA_{ijk}|^2}$. An $n_1\times n_2\times n_3$ tensor $\calA$ can be transformed to an $(n_1n_3)\times (n_2n_3)$ block diagonal matrix $A$ whose blocks are the frontal slices $A^{(1)},\cdots,A^{(n_3)}$. Both this transform and its inverse are denoted by $\texttt{bdiag}$, and the meaning can be understood according to the type of the input.

Discrete Fourier transform (DFT) and inverse discrete Fourier transform (IDFT) are essential to the definitions in Section \ref{tp and tsvd}. We use $\fft$ and $\ifft$ to denote DFT and IDFT applying to each tube of a 3-way tensor. We define $\overline{\calA}=\fft(\calA,3)$, and it is obvious that $\calA = \ifft(\overline{\calA},,3)$. Furthermore, we use $\overline{A}=\texttt{bdiag}(\overline{\calA})$ to denote the block diagonal matrix whose blocks are frontal slices of $\overline{\calA}$. With a little abuse of terminology, we say $\calA$ is in original domain and $\overline{\calA}$ (or equivalently $\overline{A}$) is in transformation domain or Fourier domain.

\subsection{T-Product and T-SVD}\label{tp and tsvd}

\begin{defn}[\bf{T-product}]\cite{kilmer2011,TNNPAMI}\label{tp}
Suppose $\calA\in\bbC^{n_1\times m\times n_3}$ and $\calB\in \bbC^{m\times n_2\times n_3}$, then the t-product $\calA*\calB\in\bbC^{n_1\times n_2\times n_3}$ is defined as 
\begin{equation}
\calA*\calB = \textnormal{\ifft}(\texttt{bdiag}(\overline{A}\overline{B}),3)
\end{equation}
\end{defn}
Note that Definition \ref{tp} is different from \cite{kilmer2011,TNNPAMI} in form, but it is equivalent to the standard definitions. The reason why we choose this form is to avoid some cumbersome notations and better reveal the relationship between original domain and transformation domain. We may regard t-product as transforming the tensors by DFT, then multiplying corresponding frontal slices in Fourier domain, and finally transforming the result back to original domain by IDFT. It has been proved in \cite{kilmer2011,TNNPAMI} that $\calC=\calA*\calB$ is equivalent to $\overline{C}=\overline{A}\overline{B}$.

Before we introduce T-SVD, we need some further definitions, which are direct extensions of the corresponding definitions in the matrix case.

\begin{defn}[\bf{Conjugate transpose}]\cite{kilmer2011,TNNPAMI} Suppose $\calA \in \bbC^{n_1\times n_2\times n_3}$, the conjugate transpose of $\calA$ is denoted by $\calA^*\in\bbC^{n_2\times n_1\times n_3}$ whose first frontal slice equals to $\big(A^{(1)}\big)^*$ and whose $k-$th frontal slice ($k=2,3\cdots,n_3$) equals to $\big(A^{(n_3+2-k)}\big)^*$.
\end{defn}

\begin{defn}[\bf{Identity tensor}]\cite{kilmer2011,TNNPAMI}
The identity tensor $\calI\in \bbR^{n\times n\times n_3}$ is the tensor whose first frontal slice is the $n\times n$ identity matrix and whose other slices are zeros.
\end{defn}

\begin{defn}[\bf{Orthogonal tensor}]\cite{kilmer2011,TNNPAMI}
A tensor $\calQ\in \bbR^{n\times n\times n_3}$ is orthogonal if $\calQ*\calQ^*=\calQ^**\calQ=\calI$.
\end{defn}

\begin{defn}[\bf{F-diagonal}] \cite{kilmer2011,TNNPAMI}
A tensor is said to be f-diagonal if its every frontal slice is a diagonal matrix.
\end{defn}

\begin{thm}[\bf{T-SVD}]\cite{kilmer2011,TNNPAMI}
Suppose $\calA\in\bbR^{n_1\times n_2\times n_3}$. Then there exists tensors $\calU\in \bbR^{n_1\times n_1\times n_3},\calV\in \bbR^{n_2\times n_2\times n_3}$ and $\calS \in \bbR^{n_1\times n_2\times n_3}$ such that $\calA = \calU*\calS*\calV^*$. Furthermore, $\calU$ and $\calV$ are orthogonal, while $\calS$ is f-diagonal.
\end{thm}

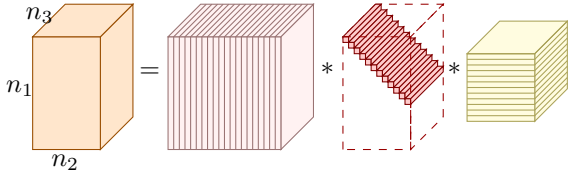
\begin{figure}[htb]
\centering
\newcommand{\none}{2}
\newcommand{\ntwo}{1.2}
\newcommand{\nthree}{1.5}

\begin{tikzpicture}[scale=0.75]
\begin{scope}[shift={(-2.4,0)}]
\coordinate (O) at (0,0,0);
\coordinate (A) at (0,\none,0);
\coordinate (B) at (0,\none,\nthree);
\coordinate (C) at (0,0,\nthree);
\coordinate (D) at (\ntwo,0,0);
\coordinate (E) at (\ntwo,\none,0);
\coordinate (F) at (\ntwo,\none,\nthree);
\coordinate (G) at (\ntwo,0,\nthree);

\draw[orange!60!black,fill=orange!20] (D) -- (E) -- (F) -- (G) -- cycle;
\draw[orange!60!black,fill=orange!20] (C) -- (B) -- (F) -- (G) -- cycle;
\draw[orange!60!black,fill=orange!20] (A) -- (B) -- (F) -- (E) -- cycle;

\draw (0,-0.8) node {$n_2$};
\draw (-0.8,0.5) node {$n_1$};
\draw (-0.45,1.8) node {$n_3$};
\end{scope}

\draw (-0.9,0.8) node {=};

\begin{scope}
\coordinate (O) at (0,0,0);
\coordinate (A) at (0,\none,0);
\coordinate (B) at (0,\none,\nthree);
\coordinate (C) at (0,0,\nthree);
\coordinate (D) at (\none,0,0);
\coordinate (E) at (\none,\none,0);
\coordinate (F) at (\none,\none,\nthree);
\coordinate (G) at (\none,0,\nthree);

\draw[pink!60!black,fill=pink!20] (D) -- (E) -- (F) -- (G) -- cycle;
\draw[pink!60!black,fill=pink!20] (C) -- (B) -- (F) -- (G) -- cycle;
\draw[pink!60!black,fill=pink!20] (A) -- (B) -- (F) -- (E) -- cycle;

\foreach \x in {0.1,0.2,...,1.9}{
	\coordinate (v1) at (\x,0,\nthree);
	\coordinate (v2) at (\x,\none,\nthree);
	\coordinate (v3) at (\x,\none,0);
	\draw[pink!60!black](v1)--(v2)--(v3);
}
\end{scope}

\draw (2.25,0.8) node {*};

\begin{scope}[shift={(3.1,0)}]
\coordinate (O) at (0,0,0);
\coordinate (A) at (0,\none,0);
\coordinate (B) at (0,\none,\nthree);
\coordinate (C) at (0,0,\nthree);
\coordinate (D) at (\ntwo,0,0);
\coordinate (E) at (\ntwo,\none,0);
\coordinate (F) at (\ntwo,\none,\nthree);
\coordinate (G) at (\ntwo,0,\nthree);

\draw[dashed,red!60!black] (D) -- (E) -- (F) -- (G) -- (D);
\draw[dashed,red!60!black] (C) -- (B) -- (F) -- (G) -- (C);
\draw[dashed,red!60!black] (A) -- (B) -- (F) -- (E) --(A);

\foreach \x in {0.1,0.2,...,1.2}{
\coordinate (O) at (\x-0.1,\none-\x,0);
\coordinate (A) at (\x,\none-\x,0);
\coordinate (D) at (\x-0.1,\none-\x+0.1,0);
\coordinate (E) at (\x,\none-\x+0.1,0);

\coordinate (B) at (\x,\none-\x,\nthree);
\coordinate (C) at (\x-0.1,\none-\x,\nthree);
\coordinate (F) at (\x,\none-\x+0.1,\nthree);
\coordinate (G) at (\x-0.1,\none-\x+0.1,\nthree);
\draw[red!50!black,fill=red!20] (D) -- (E) -- (F) -- (G) -- cycle;
\draw[red!50!black,fill=red!20] (C) -- (B) -- (F) -- (G) -- cycle;
\draw[red!50!black,fill=red!20] (A) -- (B) -- (F) -- (E) -- cycle;

}
\coordinate (B) at (0,\none,\nthree);
\coordinate (F) at (\ntwo,\none,\nthree);
\coordinate (G) at (\ntwo,0,\nthree);
\draw[dashed,color=red!60!black] (F) -- (G);
\draw[dashed,color=red!60!black] (B) -- (F) ;

\end{scope}

\draw (4.5,0.8) node {*};

\begin{scope}[shift={(5.3,0.5)}]
\coordinate (O) at (0,0,0);
\coordinate (A) at (0,\ntwo,0);
\coordinate (B) at (0,\ntwo,\nthree);
\coordinate (C) at (0,0,\nthree);
\coordinate (D) at (\ntwo,0,0);
\coordinate (E) at (\ntwo,\ntwo,0);
\coordinate (F) at (\ntwo,\ntwo,\nthree);
\coordinate (G) at (\ntwo,0,\nthree);

\draw[yellow!60!black,fill=yellow!15] (D) -- (E) -- (F) -- (G) -- cycle;
\draw[yellow!60!black,fill=yellow!15] (C) -- (B) -- (F) -- (G) -- cycle;
\draw[yellow!60!black,fill=yellow!15] (A) -- (B) -- (F) -- (E) -- cycle;

\foreach \x in {0.1,0.2,...,1.2}{
	\coordinate (v1) at (\ntwo,\x,0);
	\coordinate (v2) at (\ntwo,\x,\nthree);
	\coordinate (v3) at (0,\x,\nthree);
	\draw[yellow!60!black](v1)--(v2)--(v3);
}
\end{scope}
\end{tikzpicture}
\caption{An illustration of the t-SVD of an $n_1\times n_2\times n_3$ tensor.}
\label{fig:tsvd}
\end{figure}

Note that $\calA = \calU*\calS*\calV^*$ in original domain is equivalent to $\overline{A}=\overline{U}\overline{S}\overline{V}^*$ in Fourier domain. Intuitively, we can obtain the T-SVD of $\calA$ by calculating SVD of each frontal slice $\overline{A}^{(k)}$ in frequency domain, \ie, $\overline{A}^{(k)}=\overline{U}^{(k)}\overline{S}^{(k)}(\overline{V}^{(k)})^*$, then transforming $\overline{\calU},\overline{\calS},\overline{\calV}$ to original domain by IDFT. However, as indicated in \cite{TNNPAMI}, this method may result in complex entries due to non-uniqueness of matrix SVD. We omit the detailed algorithm for calculating T-SVD due to space limit, and refer to \cite{TNNPAMI} for further discussions.

The concept of rank for tensors is very complicated. There are various definitions of tensor rank \cite{kolda,cichocki2015tensor,sidiropoulos2017tensor}, and most of them are NP-complete. The rank of a matrix is equivalent to the number of its non-zero singular values, and we often use nuclear norm (the sum of all singular values) as a surrogate function for matrix rank. Intuitively, we may extend the concept of the nuclear norm to the tensor case, and the extension may be a reasonable surrogate for tensor rank.

\begin{defn}[\bf{Tensor nuclear norm}]\cite{kilmer2011,TNNPAMI}\label{tnn}
Let $\calA =\calU*\calS*\calV^*$ be the t-SVD of $\calA$, the nuclear norm of $\calA$ is defined as $\|\calA\|_*=\sum_i\calS(i,i,1)$.
\end{defn}

It has been proved in \cite{TNNPAMI} that Definition \ref{tnn} is the convex envelope of tensor average rank. Besides, the tensor nuclear norm is the dual norm of the tensor spectral norm, which is consistent with the matrix case. At first glance, the definition above may be a little amazing since only the first frontal slice of $\calS$ is used. According to the definition of IDFT, we have $\calS(i,i,1) =\frac{1}{n_3} \sum_{k}\overline{\calS}(i,i,k)$. Thus, in the transformation domain, the tensor nuclear norm is equal to the sum of all singular values of all frontal slices up to a constant factor.

\subsection{Non-convex Penalties: SCAD and MCP}\label{scadandmcp}
As indicated in \cite{SCAD}, an ideal penalty function should result in an estimator with three properties: unbiasedness, sparsity and continuity. Smoothly clipped absolute deviation (SCAD) was proposed in \cite{SCAD} to improve the properties of the $\ell_1$ penalty, which does not satisfies the three properties simultaneously.
\begin{defn}[\bf{SCAD}]\cite{SCAD} For some $\gamma>2$ and $\lambda>0$, the SCAD function is given by
\begin{equation}
\varphi^{\mathrm{SCAD}}_{\lambda,\gamma}(t)=\begin{cases}
\lambda|t|\quad &\mathrm{if } |t|\le \lambda,\\
\frac{\gamma\lambda |t|-0.5(t^2+\lambda^2)}{\gamma-1}\quad &\mathrm{if } \lambda<|t|<\gamma\lambda, \\
\frac{\gamma+1}{2}\lambda^2 \quad &\mathrm{if }|t|>\gamma\lambda.
\end{cases}
\end{equation}
\end{defn}
A continuous, nearly unbiased and accurate variable selection penalty called minimax concave penalty (MCP) was proposed in \cite{MCP}. The precise definition is given as follows.

\begin{defn}[\bf{MCP}]\cite{MCP} For some $\gamma>1$ and $\lambda>0$, the MCP function is given by
\begin{equation}
\varphi^{\mathrm{MCP}}_{\lambda,\gamma}(t)=\begin{cases}
\lambda|t|-\frac{t^2}{2\gamma}\quad&\mathrm{if }|t|<\gamma\lambda,\\
 \frac{\gamma\lambda^2}{2}\quad&\mathrm{if } |t|\ge \gamma\lambda.
\end{cases}
\end{equation}
\end{defn}

It is well known that $\ell_1-$norm penalty over-penalizes large components. However, in SCAD and MCP, the penalty remains constant once the variable is larger than a threshold. Besides, we point out that as $\gamma\rightarrow \infty$, we have $\varphi^{\mathrm{SCAD}}_{\lambda,\gamma}(t)\rightarrow \lambda|t|$ and $\varphi^{\mathrm{MCP}}_{\lambda,\gamma}(t)\rightarrow \lambda|t|$ pointwisely. Last but not least, if we restrict $t\ge 0$, or equivalently view SCAD and MCP as functions of $|t|$, then they are concave functions. In the following, we use $\varphi_{\lambda,\gamma}(t)$ to denote SCAD or MCP alternatively. 

 The effects of $\lambda$ and $\gamma$ in SCAD and MCP can be understood intuitively by considering $\varphi_{\lambda,\gamma}(t)\rightarrow\lambda|t|$. Roughly, $\lambda$ controls the relative importance of the penalty, and $\gamma$ controls how similar is $\varphi_{\lambda,\gamma}(t)$ compared with $\lambda|t|$.

\section{Theoretical Foundations}\label{sec:theoretical}
\subsection{A Novel Tensor Sparsity Measure}
The $\ell_1-$norm has been widely used as a sparsity measure in statistics, machine learning and computer vision. For tensors, the tensor $\ell_1-$norm plays a vital role in TRPCA \cite{zhao2015novel,TNN,TNNPAMI}. However, $\ell_1-$norm penalty over-penalizes larger entries and may result in biased estimator. Therefore, we propose to use SCAD or MCP instead of the $\ell_1-$norm penalty. The novel tensor sparsity measure is defined as
\begin{equation}
\Phi_{\lambda,\gamma}(\calA) = \sum_{i=1}^{n_1}\sum_{j=1}^{n_2}\sum_{k=1}^{n_3}\varphi_{\lambda,\gamma}(\calA_{ijk}).
\end{equation}
Here, we may set $\varphi_{\lambda,\gamma}$ to be $\varphi^{\mathrm{SCAD}}_{\lambda,\gamma}$ or $\varphi^{\mathrm{MCP}}_{\lambda,\gamma}$. We have the following properties.

\begin{prop} For $\calA\in\bbR^{n_1\times n_2\times n_3}$, $\Phi_{\lambda,\gamma}(\calA)$ satisfies:
\begin{enumerate}[(i)]
\item{} $\Phi_{\lambda,\gamma}(\calA)\ge 0$ with the equality holds iff $\calA=0$;
\item{} $\Phi_{\lambda,\gamma}(\calA)$ is concave with respect to $|\calA|$;
\item{} $\Phi_{\lambda,\gamma}(\calA)$ is increasing in $\gamma$, $\Phi_{\lambda,\gamma}(\calA)\le\lambda \|\calA\|_1$ and $\lim_{\gamma\rightarrow \infty}\Phi_{\lambda,\gamma}(\calA)=\lambda\|\calA\|_1$.
\end{enumerate}
\end{prop}

\subsection{A Novel Tensor Rank Penalty}
In this part, we always assume $\lambda=1$. Similar to tensor nuclear norm, we can apply SCAD or MCP to the singular values of a tensor. However, this may result in difficulty in optimization algorithms. Instead, we propose to apply penalty function to all singular values in Fourier domain. More precisely, suppose $\calA$ has t-SVD $\calA = \calU*\calS*\calV^*$, we define the $\gamma-$norm of $\calA$ as 
\begin{equation}
\|\calA\|_{\gamma} =\frac{1}{n_3}\sum_{i,k} \varphi_{1,\gamma}(\overline{\calS}(i,i,k)).
\end{equation}
The tensor $\gamma-$norm enjoys the following properties.
\begin{prop} For $\calA\in\bbR^{n_1\times n_2\times n_3}$, suppose $\calA$ has t-SVD $\calA = \calU*\calS*\calV^*$, then $\|\calA\|_{\gamma}$ satisfies:
\begin{enumerate}[(i)]
\item{} $\|\calA\|_{\gamma}\ge 0$ with equality holds iff $\calA=0$;
\item{} $\|\calA\|_{\gamma}$ is increasing in $\gamma$, $\|\calA\|_{\gamma}\le \|\calA\|_*$ and $\lim_{\gamma\rightarrow\infty}\|\calA\|_{\gamma}=\|\calA\|_*$;
\item{} $\|\calA\|_{\gamma}$ is concave with respect to $\{\overline{\calS}(i,i,k)\}_{i,k}$;
\item{} $\|\calA\|_{\gamma}$ is orthogonal invariant, \ie, for any orthogonal tensors $\mathcal{P}\in\bbR^{n_1\times n_1\times n_3},\calQ\in\bbR^{n_2\times n_2\times n_3}$, we have $\|\mathcal{P}*\calA*\calQ\|_{\gamma}=\|\calA\|_{\gamma}$.
\end{enumerate}
\end{prop}

\subsection{Generalized Thresholding Operators}
We will use majorization minimization algorithm in Section \ref{sec:tc} and \ref{sec:trpca}. In this part, we derive some properties that are vital to MM algorithm based on the concavity of SCAD and MCP. As mentioned in Section \ref{scadandmcp}, SCAP and MCP are continuous differentiable concave functions restricted on $[0,\infty)$, thus we can bound $\varphi_{\lambda,\gamma}(t)$ by its first-order Taylor expansion $\varphi_{\lambda,\gamma}(t_0)+\varphi'_{\lambda,\gamma}(t_0)(t-t_0)$. This observation leads to the following theorem.

\begin{thm}\label{upperbound}
We can view $\Phi_{\lambda,\mu}(\calX)$ as a function of $|\calX|$, and $\|\calX\|_{\gamma}$ as a function of $\{\overline{\calS}(i,i,k)\}_{i,k}$. For any $\calX^{\textnormal{old}}$, let
{\small\begin{equation}
\begin{aligned}
Q_{\lambda,\gamma}(\calX|\calX^{\textnormal{old}})&=\Phi_{\lambda,\gamma}(\calX^{\textnormal{old}})+\sum_{i,j,k}\varphi'_{\lambda,\gamma}(|\calX^{\textnormal{old}}_{ijk}|)(|\calX_{ijk}|-|\calX^{old}_{ijk}|),\\
Q_{\gamma}(\calX|\calX^{\textnormal{old}})&= \|\calX^{\textnormal{old}}\|_{\gamma}+\frac{1}{n_3}\sum_{i,k}\varphi_{1,\gamma}'(\overline{\calS}^{\old}_{iik})(\overline{\calS}_{iik}-\overline{\calS}^{\textnormal{old}}_{iik}),
\end{aligned}
\end{equation}}
then 
\begin{equation}
\begin{aligned}
Q_{\lambda,\gamma}(\calX^{\mathrm{old}}|\calX^{\mathrm{old}})=\Phi_{\lambda,\gamma}(\calX^{\mathrm{old}}),\Phi_{\lambda,\gamma}(\calX)\le Q_{\lambda,\gamma}(\calX|\calX^{\textnormal{old}}),\\
Q_{\gamma}(\calX^{\mathrm{old}}|\calX^{\mathrm{old}})=\|\calX^{\mathrm{old}}\|_{\gamma},\|\calX\|_{\gamma}\le Q_{\gamma}(\calX|\calX^{\textnormal{old}}).
\end{aligned}
\end{equation}
\end{thm}

Due to the concavity of $\Phi_{\lambda,\gamma}(\calX)$ and $\|\calX\|_{\gamma}$, optimization problems involving $\Phi_{\lambda,\gamma}(\calX)$ and $\|\calX\|_{\gamma}$ are generally extremely difficult to solve. However, optimizing upper bounds given in Theorem \ref{upperbound} instead is relatively easy. It's well-known that soft thresholding operator $\calT_{\lambda}(z)=\textnormal{sgn}(z)[|z|-\lambda]_+$ is the proximal operator of $\ell_1-$norm. In the following, we introduce generalized thresholding operators based on $\calT_{\lambda}$, then derive the proximal operators of $Q_{\lambda,\gamma}(\calX|\calX^{\old})$ and $Q_{\gamma}(\calX|\calX^{\old})$.

\begin{defn}[\bf{Generalized soft thresholding}]
Suppose $\calX,\calW\in\bbR^{n_1\times n_2\times n_3}$, the generalized soft thresholding operator is defined as
\begin{equation}
[\calT_{\calW}(\calX)]_{ijk} = \calT_{\calW_{ijk}}(\calX_{ijk}).
\end{equation}
\end{defn}

\begin{thm}\label{thm:generalizedthresholding}
For $\forall\mu>0$, let $\calW_{ijk}=\varphi_{\lambda,\gamma}'(|\calX^{\old}_{ijk}|)/\mu$, then 
\begin{equation}
\calT_{\calW}(\calY)=\arg\min_{\calX}Q_{\lambda,\gamma}(\calX|\calX^{\old})+\frac{\mu}{2}\|\calX-\calY\|_F^2.
\end{equation}
\end{thm}

\begin{defn}[\bf{Generalized t-SVT}]Suppose a 3-way tensor $\calY$ has t-SVD $\calY=\calU*\calS*\calV^*$, $\calW$ is a tensor with the same shape of $\calY$, the generalized tensor singular value thresholding operator is defined as 
\begin{equation}
\calD_{\calW}(\calY)=\calU*\tilde{\calS}*\calV^*,
\end{equation}
where $\tilde{\calS}=\textnormal{\ifft}(\calT_{\calW}(\calS),3)$.
\end{defn}

\begin{thm}\label{thm:generalizedtSVT}
For $\forall \mu>0$, let $\calW_{ijk}=\delta_{i}^j\varphi_{1,\gamma}'(\overline{\calS}^{\old}_{iik})/\mu$ where $\delta_i^j$ is the Kronecker symbol, then
\begin{equation}
\calD_{\calW}(\calY)=\arg\min_{\calX}Q_{\gamma}(\calX|\calX^{\old})+\frac{\mu}{2}\|\calX-\calY\|_F^2.
\end{equation}
\end{thm}

\section{Proposed Non-convex Tensor Recovery Models and Algorithms}\label{sec:models}
\subsection{Non-convex Tensor Completion}\label{sec:tc}

\begin{algorithm}[t]  
\caption{MM algorithm for non-convex low-rank tensor completion}  
\label{alg:tcmm}  
{\bf Input:} $\Omega,\calO$\\
{\bf Hyper parameters:} $\gamma,\mu_0,\rho,\mu_{\text{max}}$\\
\begin{algorithmic}[1]
\STATE Initialize $\calX^{\old}=\calX^0$ by $\calO\circledast\Omega$ or other tensor completion algorithm.
\WHILE{ not converged}
\STATE Calculate $\overline{\calS}^{\old}$ and set $\calW^t_{ijk}=\delta_{i}^j\varphi_{1,\gamma}'(\overline{\calS}^{\old}_{iik})/\mu$.
\STATE Initialize $\calX_0^t=\calX^t,\calY_0^t=0$.
\WHILE{not converged}
	\STATE $\calM^{t}_{l+1}=\calD_{\calW^t}(\calX^t_l-\frac{1}{\mu_l}\calY^t_l)$
	\STATE $\calX^{t}_{l+1}=(\calM^t_{l+1}+\frac{1}	{\mu_l}\calY^t_l)\circledast(1-\Omega)+\calO\circledast\Omega$
	\STATE $\calY^{t}_{l+1} = \calY^t_l+\mu_l(\calM^t_{l+1}-\calX^t_{l+1})$
	\STATE $\mu_{l+1}=\min(\rho\mu_l,\mu_{\text{max}})$
\ENDWHILE
\STATE Update $\calX^{t+1}$ by the result of inner iteration
\STATE Set $\calX^{\old}=\calX^{t+1}$
\ENDWHILE
\end{algorithmic}  
\end{algorithm}  

Given a partially observed tensor $\calO\in \bbR^{n_1\times n_2\times n_3}$, tensor completion task aims to recover the full tensor $\calX$ which coincides with $\calO$ in the observed positions. Suppose the observed positions are indexed by $\Omega$, \ie, $\Omega_{ijk}=1$ denotes the $(i,j,k)-$th element is observed while $\Omega_{ijk}=0$ denotes the $(i,j,k)-$th element is unknown. Based on low rank assumption, tensor completion can be modeled as 
\begin{equation}\label{tcmodel1}
\min_{\calX} \textnormal{rank}(\calX)\quad \text{s.t. }
\calO_{\Omega}=\calX_{\Omega}.
\end{equation}
Since the concept of rank is very complicated for tensors, many types of tensor rank or surrogate functions can be used in Equation \ref{tcmodel1}. Here, we use the proposed tensor $\gamma-$norm, 
\begin{equation}\label{tcmodel}
\min_{\calX} \|\calX\|_{\gamma}\quad \text{s.t. }\calO_{\Omega}=\calX_{\Omega}.
\end{equation}
We can set $\varphi_{1,\gamma}$ in Equation (\ref{tcmodel}) to be SCAD or MCP. In the following we refer these non-convex tensor completion models as $\text{LRTC}_{\mathrm{scad}}$ and $\text{LRTC}_{\mathrm{mcp}}$ respectively.

We apply majorization minimization algorithm to solve problem (\ref{tcmodel}). 
 Given $\calX^{\old}$, we minimize the upper bound of $\|\calX\|_{\gamma}$ given in Theorem \ref{upperbound}, which leads to
\begin{equation}\label{eq:mmtc}
\min_{\calX} Q_{\gamma}(\calX|\calX^{\old})\quad\text{s.t. } \calO_{\Omega}=\calX_{\Omega}.
\end{equation}
Problem (\ref{eq:mmtc}) is convex, thus we can use ADMM to solve it. Introducing auxiliary variable $\calM$ and let $D$ be the feasible domain $\{\calX|\calO_{\Omega}=\calX_{\Omega}\}$, then Equation (\ref{tcmodel}) is equivalent to 
\begin{equation}\label{eq:tcadmm}
\min_{\calX\in D} Q_{\gamma}(\calM|\calX^{\old})\quad\text{s.t. } \calM=\calX.
\end{equation}
Problem (\ref{eq:tcadmm}) is easy to solve by standard ADMM iterations, and the derivation of ADMM steps are omitted here for page limit and included in the supplemental matrial. Note that ADMM is the inner loop, after the ADMM converges we should update $\calX^{\old}$ and repeat ADMM iterations again. Detailed algorithm is described in Algorithm \ref{alg:tcmm}. We have the following theoretical guarantee for this algorithm.

\begin{thm}
	The iteration sequence generated by $\calX^{t+1}\in \arg\min_{\calO_{\Omega}=\calX_{\Omega}}Q_{\gamma}(\calX|\calX^t)$ is non-increasing, \ie, $\|\calX^{t+1}\|_{\gamma}\le \|\calX^t\|_{\gamma}$ and converges to some $Q^*$. Besides, there exists a subsequence $\{\calX^{i_k}\}_{k=1}^{\infty}$ converges to a minimal point $\calX_*$ of $\|\calX\|_{\gamma}$ on $\{\calX|\calO_{\Omega}=\calX_{\Omega}\}$.
\end{thm}

\subsection{Non-convex Tensor Robust PCA}\label{sec:trpca}

\begin{algorithm}[t]  
\caption{MM algorithm for tensor RPCA}  
\label{alg:trpcamm}  
{\bf Input:} $\calX$\\
{\bf Hyper parameters:} $\gamma_1,\gamma_2,\mu_0,\rho,\mu_{\max}$\\
\begin{algorithmic}[1]
\STATE Initialize $\calL^0,\calE^0$ by other tensor RPCA algorithm
\STATE Initialize $\calY^0$ by random guess
\WHILE{ not converged}
\STATE Calculate t-SVD of $\calL^{\old}=\calU*\calS^{\old}*\calV^*$
\STATE Set $\calZ^t_{ijk}=\delta_{i}^j\varphi_{1,\gamma_1}'(\overline{\calS}^{\old}_{iik})/\mu$
\STATE Set $\calW^t_{ijk}=\varphi_{\lambda,\gamma_2}'(\calE^{\old}_{ijk})/\mu$
\WHILE{ not converged}
	\STATE $\calL^{t}_{l+1}=\calD_{\calZ^t}(\calX-(\calE^t_l+\frac{1}{\mu_l}\calY^t_l))$
	\STATE $\calE^{t}_{l+1}=\calT_{\calW^t}(\calX-(\calL^t_{l+1}+\frac{1}{\mu_l}\calY_l))$
	\STATE $\calY^{t}_{l+1} = \calY^t_l+\mu_l\calL^t_{l+1}+\calE^t_{l+1}-\calX)$
	\STATE $\mu_{l+1}=\min(\rho\mu_l,\mu_{\max})$
\ENDWHILE
\STATE Update $\calL^{t+1},\calE^{t+1}$ by the result of inner iteration
\STATE Set $\calL^{\old}=\calL^{t+1},\calE^{\old}=\calE^{t+1}$
\ENDWHILE
\end{algorithmic}  
\end{algorithm}  

Given a tensor $\calX$, the goal of robust PCA is to decompose $\calX$ into two parts: low-rank tensor $\calL$ and sparse tensor $\calE$. This problem can be formulated as 
\begin{equation}
\min_{\calL,\calE}\text{rank}(\calL)+\|\calE\|_{0}\quad\text{s.t. }\calL+\calE=\calX.
\end{equation}
Apply the proposed novel sparsity measure and tensor $\gamma-$norm, we obtain
\begin{equation}\label{eq:trpcamodel}
\min_{\calL,\calE}\|\calL\|_{\gamma_1}+\Phi_{\lambda,\gamma_2}(\calE)\quad\text{s.t. }\calL+\calE=\calX.
\end{equation}
We may set $\varphi_{\lambda,\gamma}$ to be SCAP or MCP in Equation (\ref{eq:trpcamodel}), and refer them as $\text{TRPCA}_{\mathrm{scad}}$ and $\text{TRPCA}_{\mathrm{mcp}}$ respectively.

We apply majorization minimization algorithm to solve problem (\ref{eq:trpcamodel}). 
 Given $\calL^{\old},\calE^{\old}$, we minimize the upper bound of $\|\calL\|_{\gamma_1}+\Phi_{\lambda,\gamma_2}(\calE)$ given in Theorem \ref{upperbound},
\begin{equation}
\min_{\calL,\calE}Q_{\gamma_1}(\calL|\calL^{\old})+Q_{\lambda,\gamma_2}(\calE|\calE^{\old})\quad\text{s.t. }\calL+\calE=\calX.
\end{equation}
This problem is also easy to solve by ADMM. The sub-problem of updating $\calL$ and $\calE$ has closed-form solutions according to Theorem \ref{thm:generalizedthresholding} and Theorem \ref{thm:generalizedtSVT}. We describe the detailed algorithm in Algorithm \ref{alg:trpcamm}. We have the following theoretical result for this algorithm.

\begin{thm}
	The iteration sequence generated by 
	\[(\calL^{t+1},\calE^{t+1})\in\arg\min_{\calL+\calE=\calX} Q_{\gamma_1}(\calL|\calL^t)+Q_{\lambda,\gamma_2}(\calE|\calE^t)\]
	is non-increasing, \ie, $\|\calL^{t+1}\|_{\gamma_1}+\Phi_{\lambda,\gamma_2}(\calE^{t+1}) \le \|\calL^{t}\|_{\gamma_1}+\Phi_{\lambda,\gamma_2}(\calE^{t})$ and converges to some $Q^*$. Besides, there exists a subsequence $\{(\calL^{i_k},\calE^{i_k})\}_{k=1}^{\infty}$ converges to a minimal point $(\calL_*,\calE_*)$ of $\|\calL\|_{\gamma_1}+\Phi_{\lambda,\gamma_2}(\calE)$ on $\{(\calL,\calE)|\calL+\calE=\calX\}$.
\end{thm}

\section{Experiments}\label{sec:experiments}
\begin{figure*}[t]
	\centering
	\includegraphics[width=0.9\textwidth]{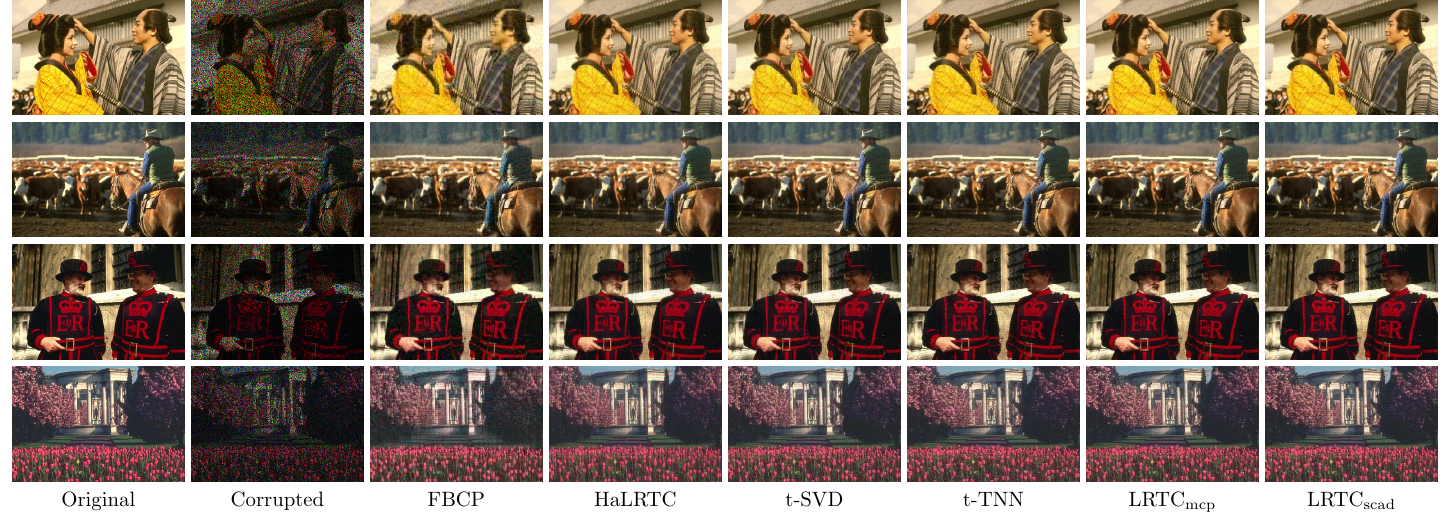}
	\caption{Tensor completion performance comparison on example images.}
	\label{fig:tc_visual}
\end{figure*}

\subsection{Datasets and Experimental Settings}
We evaluate the effectiveness of the proposed non-convex tensor completion and tensor RPCA algorithms on Berkeley Segmentation 500 Dataset (BSD 500) \cite{bsd500} and Natural Scenes 2002 Dataset (NS 2002) \cite{naturalscenes2002}. Berkeley Segmentation 500 Dataset consists of 500 natural images, and Natural Scenes 2002 Dataset contains 8 hyperspectral images with 31 bands sampled from 410nm to 710nm at 10nm intervals. All the hyperspectral images are downsampled by factor 2. We employ Mean Square Error (MSE), Peak Signal-to-Noise Ratio (PSNR), Feature SIMilarity (FSIM) \cite{FSIM}, Erreur Relative Globale Adimensionnelle de Synth\`ese (ERGAS)\cite{ERGAS} and Spectral Angle Mapper (SAM) \cite{SIPS,SAM} as performance evaluation indexes. Smaller MSE, ERGAS, SAM and larger PSNR, FSIM indicates the result is better.

There are some practical issues to clarify about Algorithm \ref{alg:tcmm} and Algorithm \ref{alg:trpcamm}. First, the hyper-parameters $\mu_0,\rho,\mu_{\text{max}}$ are introduced to accelerate the convergence speed. The inner ADMM iteration is always convergent regardless of the settings of these parameters, but the speed of convergence is different. In practice, we find setting $\mu_0=1,\rho=1.1,\mu_{\text{max}}=1e10$ results in fast convergence. Second, the initialization of $\calX^0$ and $\calL,\calE$ is very important, since a good starting position usually leads to better final result in non-convex optimization problems. We suggest initializing $\calX^0$ or $\calL^0,\calE^0$ by other tensor completion or tensor RPCA methods (such as TRPCA \cite{TNNPAMI}). Last but not least, it usually takes a long time for the outer iteration to converge. In practice, it's not necessary to wait for convergence. Instead, we can iterate the outer loop for fixed times.

\begin{table*}[htb]
\centering
\caption{Tensor completion performances evaluation on natural images under varying sampling rates.}
\resizebox{0.925\linewidth}{!}{
\begin{tabular}{cccccccccccccc}
\toprule
\multirow{2}{*}{Method} & \multicolumn{3}{c}{20\%} & \multicolumn{3}{c}{40\%} & \multicolumn{3}{c}{60\%} & \multicolumn{3}{c}{80\%} & \multirow{2}{*}{time (s)} \\ \cmidrule(lr){2-4} \cmidrule(lr){5-7} \cmidrule(lr){8-10} \cmidrule(lr){11-13}
                        &PSNR & SSIM   & FSIM   & PSNR   & SSIM   & FSIM   & PSNR    & SSIM   & FSIM   & PSNR    & SSIM   & FSIM   &                       \\ \midrule

SiLRTC &$23.59$ & $0.798$ & $0.822$ &$27.987$ & $0.899$ & $0.915$ &$32.24$ & $0.951$ & $0.964$ &$37.47$ & $\bf0.977$ & $0.988$ & $19.95$ \\ 
 HaLRTC  &$23.82$ & $0.797$ & $0.828$ &$28.39$ & $0.902$ & $0.920$ &$33.038$ & $0.953$ & $0.968$ &$39.27$ & $\bf0.978$ & $0.991$ & $31.32$ \\ 
FBCP  &$24.08$ & $0.668$ & $0.794$ &$26.40$ & $0.753$ & $0.837$ &$27.35$ & $0.799$ & $0.857$ &$27.71$ & $0.82$ & $0.865$ & $103.68$ \\ 
t-SVD  &$24.13$ & $0.764$ & $0.835$ &$29.703$ & $0.893$ & $0.931$ &$36.03$ & $0.950$ & $0.977$ &$45.04$ & $0.969$ & $0.992$ & $33.47$ \\ 
t-TNN  &$25.30$ & $0.841$ & $0.864$ &$30.50$ & $0.923$ & $0.943$ &$36.27$ & $0.952$ & $0.978$ &$44.14$ & $0.967$ & $0.991$ & $\bf3.03$ \\  \midrule
$\text{LRTC}_{mcp}$ &$\bf 25.70$ & $ \bf 0.845$ & $ \bf 0.869$ &$ \bf 31.06$ & $\bf0.927$ & $\bf0.946$ &$\bf36.87$ & $\bf0.959$ & $\bf0.980$ &$\bf45.46$ & $0.973$ & $\bf0.993$ & $3.79$ \\ 
$\text{LRTC}_{scad}$ &$ \bf 25.70$ & $\bf 0.844$ & $\bf 0.869$ &$ \bf 31.04$ & $\bf0.926$ & $\bf0.946$ &$\bf36.84$ & $\bf0.959$ & $\bf0.980$ &$\bf45.45$ & $0.973$ & $\bf0.993$ & $3.83$ \\  \bottomrule
\end{tabular}}
\label{table:tc_bsd}
\end{table*} 
\begin{table*}[htb]
\centering
\caption{Tensor completion performances evaluation on hyperspectral images under varying sampling rates. The unit is $10^{-4}$ for MSE.}
\resizebox{0.925\linewidth}{!}{
\begin{tabular}{cccccccccccccc}
\toprule
\multirow{2}{*}{Method} & \multicolumn{3}{c}{20\%} & \multicolumn{3}{c}{40\%} & \multicolumn{3}{c}{60\%} & \multicolumn{3}{c}{80\%} & \multirow{2}{*}{time (s)} \\ \cmidrule(lr){2-4} \cmidrule(lr){5-7} \cmidrule(lr){8-10} \cmidrule(lr){11-13}
                        &PSNR& MSE & ERGAS   & PSNR   & MSE   & ERGAS   & PSNR    & MSE   & ERGAS   & PSNR    & MSE & ERGAS   &   \\ \midrule

SiLRTC& $41.71$ & $4.70$ & $30.912$ & $45.46$ & $1.95$ & $21.524$ & $49.14$ & $0.827$ & $14.412$ & $52.86$ & $0.354$ & $10.341$ & $42.21$ \\ 
 HaLRTC  &$42.11$ & $4.53$ & $29.626$ &$45.95$ & $1.90$ & $20.556$ &$49.79$ & $0.801$ & $13.514$ &$53.67$ & $0.342$ & $9.660$ & $53.11$ \\ 
FBCP &$37.09$ & $14.47$ & $52.931$ &$43.25$ & $3.85$ & $29.318$ &$46.00$ & $2.225$ & $22.344$ &$46.67$ & $2.011$ & $20.688$ & $210.18$ \\ 
t-SVD &$41.64$ & $5.10$ & $31.835$ &$45.52$ & $2.12$ & $22.142$ &$49.42$ & $0.886$ & $14.685$ &$53.49$ & $0.365$ & $10.157$ & $224.21$ \\ 
t-TNN &$42.46$ & $3.81$ & $28.702$ &$46.07$ & $1.60$ & $20.135$ &$49.82$ & $0.667$ & $\bf 13.272$ &$53.61$ & $\bf 0.290$ & $\bf 9.515$ & $\bf 47.29$ \\ 
\midrule

$\text{LRTC}_{\mathrm{mcp}}$ &$\bf 42.91$ & $\bf 3.58$ & $\bf 27.799$ &$\bf 46.75$ & $\bf 1.51$ & $\bf 19.684$ &$\bf 50.47$ & $\bf 0.665$ & $\bf 13.293$ &$\bf 54.11$ & $\bf 0.316$ & $\bf 9.642$ & $70.95$ \\ 
$\text{LRTC}_{\mathrm{scad}}$ &$\bf  42.91$ & $\bf 3.58$ & $\bf 27.804$ &$\bf 46.76$ & $\bf 1.51$ & $\bf 19.684$ &$\bf 50.46$ & $\bf 0.666$ & $13.298$ &$\bf 54.11$ & $\bf 0.316$ & $\bf 9.642$ & $71.14$ \\ 
\bottomrule
\end{tabular}}
\label{table:tc_hsi}
\end{table*} 
\subsection{Tensor Completion Experiments}
We conduct tensor completion experiments on BSD 500 and NS 2002 to test the performances of $\text{LRTC}_{\mathrm{mcp}}$ and $\text{LRTC}_{\mathrm{scad}}$. For comparison, we also consider five competing tensor completion methods: Bayesian CP Factorization (FBCP) \cite{BCPF}, Simple Low-Rank Tensor Completion (SiLRTC) \cite{SNN}, High Accuracy Low-Rank Tensor Completion (HaLRTC) \cite{SNN}, tensor-SVD based method (t-SVD) \cite{ZeminZhang}, twist Tensor Nuclear Norm based method (t-TNN) \cite{tTNN}. 

\begin{figure*}[ht]
	\centering
	\includegraphics[width=0.9\textwidth]{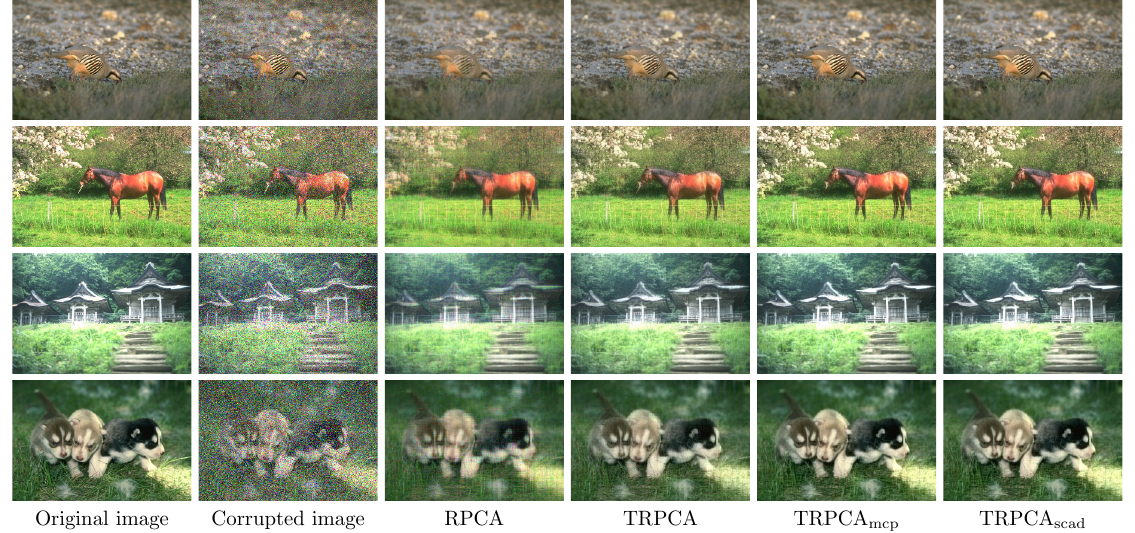}
	\caption{Tensor RPCA performance comparison on example images. From top to bottom: $p_n=0.1,0.2,0.3,0.4$.}
	\label{fig:trpca_visual}
\end{figure*}

\begin{figure*}[ht]
\centering
\includegraphics[width=0.9\textwidth]{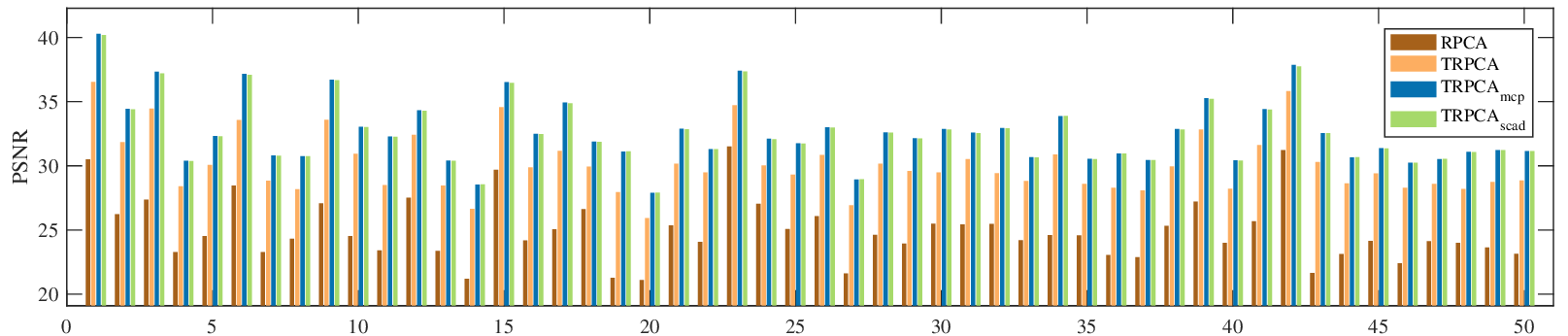}
\caption{Comparison of PSNR values obtained by RPCA, TRPCA, $\text{TRPCA}_{\mathrm{mcp}}$,$\text{TRPCA}_{\mathrm{scad}}$ on randomly selected 50 images.}
\label{fig:trpca_bsd}
\end{figure*}

{\bf Natural image inpainting.} We randomly select 200 images in BSD 500 for evaluation. For each image, pixels are randomly sampled with a sampling rate ranging from $20\%$ to $80\%$. For our $\text{LRTC}_{\mathrm{mcp}}$ and $\text{LRTC}_{\mathrm{scad}}$ models, we set $\gamma=25$, and use the result of t-TNN as initialization. To alleviate redundant computations, we apply the one-step LLA strategy \cite{onestep,wang2013nonconvex}, \ie, we run the outer loop only once instead of waiting for convergence. The average performances over selected 200 images under different sampling rates are summarized in Table \ref{table:tc_bsd}. From this table, we can see that our proposed $\text{LRTC}_{\mathrm{mcp}}$ and $\text{LRTC}_{\mathrm{scad}}$ outperform other competing methods in terms of all performance evaluation indices. As for efficiency, the proposed methods are significantly faster than FBCP, SiLRTC, HaLRTC, and t-SVD. Since the proposed methods are initialized by t-TNN, the running time is always slightly longer than t-TNN. However, the performances are improved by only one MM iteration and the extra running times are marginal. Therefore, we claim that it is necessary to introduce non-convexity in the tensor completion task. These observations demonstrate $\text{LRTC}_{\mathrm{mcp}}$ and $\text{LRTC}_{\mathrm{scad}}$ are both effective and efficient. We also give visual comparisons in Figure \ref{fig:tc_visual}.

{\bf Hyperspectral image inpainting.} We use all 8 hyperspectral images in this experiment. For each hyperspectral image, we randomly sample its elements with a sampling rate ranging from $0.2$ to $0.8$. Since the sizes of hyperspectral images are relatively large, we run the outer loop in Algorithm \ref{alg:tcmm} for 10 iterations based on t-TNN initialization. The performance comparison are shown in Table \ref{table:tc_hsi}. We have similar observations as in the natural image case: the results obtained by $\text{LRTC}_{\mathrm{mcp}}$ and $\text{LRTC}_{\mathrm{scad}}$ have lower MSE, ERGAS, and higher PSNR, indicating that the proposed methods outperform competing methods. The fastest and the slowest algorithms are different in Table \ref{table:tc_bsd} and \ref{table:tc_hsi} mainly because $n_3=3$ for natural images while $n_3=31$ for hyperspectral images and the time complexities of these algorithms depends on $n_3$ differently. Besides, we find that the results of $\text{LRTC}_{\mathrm{mcp}}$ and $\text{LRTC}_{\mathrm{scad}}$ are nearly the same since MCP and SCAD shares very similar properties. 

\subsection{Tensor RPCA Experiments}
We compare our proposed $\text{TRPCA}_{\mathrm{mcp}}$ and $\text{TRPCA}_{\mathrm{scad}}$ with matrix RPCA \cite{wright2009robust,candes2011robust} and TNN based TRPCA \cite{TNN,TNNPAMI} on both natural images and multispectral images. To apply matrix RPCA in tensor RPCA task, we simply apply matrix RPCA to each frontal slices of the corrupted tensor. 

{\bf Natural image restoration.} We first test $\text{TRPCA}_{\mathrm{mcp}}$ and $\text{TRPCA}_{\mathrm{scad}}$ on BSD 500. Each image is corrupted by salt-and-pepper noise with probability $p_n=0.1$. We set $\gamma_1=\gamma_2=20$ and run the outer loop of Algorithm \ref{alg:trpcamm} for 10 iterations. Performance comparison on randomly selected 50 images are shown in Figure \ref{fig:trpca_bsd}. From this figure, we have following observations. First, the results of TRPCA,$\text{TRPCA}_{\mathrm{mcp}}$ and $\text{TRPCA}_{\mathrm{scad}}$ are significantly better than results of RPCA. This indicates that considering tensor structure helps to improve recovery quality compared to consider each channel individually. Second, $\text{TRPCA}_{\mathrm{mcp}}$
 and $\text{TRPCA}_{\mathrm{scad}}$ obtained better performance than TRPCA, which means introducing concavity to Tensor RPCA tasks are necessary. Third, the PSNR values of $\text{TRPCA}_{\mathrm{mcp}}$ and $\text{TRPCA}_{\mathrm{scad}}$ are very similar, indicating the final result is not sensitive to the choice of non-convex penalty. We also give visual comparisons in Figure \ref{fig:trpca_visual}. Note that for the noise proportion ranging from $0.1$ to $0.4$ in Figure \ref{fig:trpca_visual}. 

\begin{table}[htb]

\caption{Tensor Robust Principal Component Analysis performances evaluation on hyper-spectral images.}
\centering
\begin{tabular}{@{\hspace{2em}}c@{\hspace{1em}}c@{\hspace{1em}}c@{\hspace{1em}}c@{\hspace{1em}}c@{\hspace{1em}}c@{\hspace{2em}}}\toprule
$p_n$ & Index &  RPCA  & TRPCA & $\text{TRPCA}_{\mathrm{mcp}}$ & $\text{TRPCA}_{\mathrm{scad}}               $ \\ \midrule
\multirow{4}{*}{0.1} & MSE & $20.113$ & $3.680$ & $\bf 3.332$ & $\bf 3.329$ \\ 
& PSNR & $41.89$ & $46.11$ & $\bf 46.57$ & $\bf 46.57$  \\ 
& ERGAS & $23.224$ & $13.203$ & $\bf 12.898$ & $\bf 12.893$ \\ 
& SAM & $0.0887$ & $0.0894$ & $\bf 0.0880$ & $\bf 0.0879$  \\  \midrule
\multirow{4}{*}{0.2} &  MSE & $21.298$ & $4.353$ & $\bf 3.837$ & $\bf 3.834$ \\ 
& PSNR & $41.31$ & $45.63$ & $\bf 46.08$ & $\bf 46.09$  \\ 
& ERGAS & $24.374$ & $14.473$ & $\bf 13.703$ & $\bf 13.706$ \\ 
& SAM & $0.1191$ & $0.1003$ & $\bf 0.0961$ & $\bf 0.0961$  \\   \midrule
\multirow{4}{*}{0.3}& MSE & $26.430$ & $6.637$ & $\bf 4.731$ & $\bf 4.726$ \\ 
& PSNR & $39.68$ & $44.13$ & $\bf 45.28$ & $\bf 45.27$  \\ 
& ERGAS & $28.956$ & $18.668$ & $\bf 15.177$ & $\bf 15.205$ \\ 
& SAM & $0.1710$ & $0.1291$ & $\bf 0.1076$ & $\bf 0.1079$  \\  \midrule
\multirow{4}{*}{0.4}& MSE & $49.858$ & $28.275$ & $\bf 6.979$ & $\bf 6.964$ \\ 
& PSNR & $36.13$ & $38.51$ & $\bf 43.16$ & $\bf 43.11$  \\ 
& ERGAS & $49.877$ & $50.354$ & $\bf 20.126$ & $\bf 20.298$ \\ 
& SAM & $0.2602$ & $0.2483$ & $\bf 0.1428$ & $\bf 0.1440$  \\  \bottomrule
\end{tabular}
\label{table:trpca_hsi}
\end{table}

{\bf Hyperspectral image restoration.} In this part, we test the proposed models on NS 2002. We add random noise to each hyperspectral image with probability $p_n$ ranging from $0.1$ to $0.4$. Here, the noise is uniformly distributed in $[0,0.1*M)$ where $M$ is the maximum absolute value of the original image. We set $\gamma_1=\gamma_2=50$, and run the outer loop for 10 iterations. The results of TRPCA are used as initialization for the proposed methods. We employ MSE, PSNR, ERGAS, and SAM as quality indexes. The results are reported in Table \ref{table:trpca_hsi}. It's easy to see that $\text{TRPCA}_{\mathrm{mcp}}$ and $\text{TRPCA}_{\mathrm{scad}}$ outperform competing methods in terms of all quality indexes. Specifically, we note that when $p_n=0.4$, \ie, the noise proportion is rather large, the proposed methods improve the results of TRPCA significantly. In these circumstances, the sparse assumption on noise may not hold. Although RPCA and TRPCA have nice exact recovery property under certain conditions, these conditions are rather strict and sometimes not satisfied. However, the proposed methods still recover the images successfully. 

\section{Conclusions}\label{sec:conclusion}
In this paper, we have presented a new non-convex tensor rank surrogate function and a new non-convex sparsity measure based on the MCP and SCAD penalty functions. Then, we have analyzed some theoretical properties of the proposed penalties. In particular, we applied the non-convex penalties in tensor completion and tensor robust principal component analysis tasks, and devised optimization algorithms based on majorization minimization. Experimental results on natural images and hyperspectral images substantiate the proposed methods outperform competing methods.

\section*{Acknowledgement}
This work was supported by the National Key Research and Development Program of China under grant 2018AAA0100205.






%

\bibliographystyle{./IEEEtran}
\bibliography{ref.bib}

\onecolumn
\section*{Appendix}
In this supplemental material, we give detailed proofs of propositions and theorems in Section 4 and derive the ADMM steps in Section 5 of our paper.

\section*{SCAD and MCP}
Although the properties of SCAD and MCP are extensively investigated in literature, we review some vital properties related to our paper here and provide proof for completeness.
\begin{defn}[\bf{SCAD}] For some $\gamma>1$ and $\lambda>0$, the SCAD function is given by
\begin{equation}
\varphi^{\mathrm{SCAD}}_{\lambda,\gamma}(t)=\begin{cases}
\lambda|t|\quad &\mathrm{if } |t|\le \lambda,\\
\frac{\gamma\lambda |t|-0.5(t^2+\lambda^2)}{\gamma-1}\quad &\mathrm{if } \lambda<|t|\le\gamma\lambda, \\
\frac{\gamma+1}{2}\lambda^2 \quad &\mathrm{if }|t|>\gamma\lambda.
\end{cases}
\end{equation}
\end{defn}

\begin{defn}[\bf{MCP}]For some $\gamma>1$ and $\lambda>0$, the MCP function is given by
\begin{equation}
\varphi^{\mathrm{MCP}}_{\lambda,\gamma}(t)=\begin{cases}
\lambda|t|-\frac{t^2}{2\gamma}\quad&\mathrm{if }|t|<\gamma\lambda,\\
 \frac{\gamma\lambda^2}{2}\quad&\mathrm{if } |t|\ge \gamma\lambda.
\end{cases}
\end{equation}
\end{defn}

We use $\NC$ to denote $\SCAD$ or $\MCP$ alternatively, then we have the following properties:
\begin{prop} 
\begin{enumerate}[(i)]
\item $\NC\ge 0$ and $\NC= 0$ if and only if $t=0$.
\item For fixed $t$ and $\lambda$, $\NC$ is increasing in $\gamma$.
\item As $\gamma\rightarrow \infty$, $\NC\rightarrow \lambda|t|$.
\item When restricted on $t\ge 0$, $\NC$ is concave.
\end{enumerate}
\end{prop}

\begin{proof}Note that $\varphi_{\lambda,\gamma}(t)$ is an even function, thus we only consider the case $t\ge 0$. 
\begin{enumerate}[(i)]
\item 
For SCAD, if $0\le t\le \lambda$ or $t>\gamma\lambda$, $\SCAD\ge 0$ since $\lambda>0$ and $\gamma>0$. If $\lambda<t\le \gamma\lambda$, the minimum of $\SCAD$ is attained at $\lambda$, which equals to $\lambda^2$. Therefore, $\SCAD\ge 0$. 

For MCP, if $t>\lambda\gamma$  $\MCP\ge 0$ since $\lambda>0,\gamma>0$. If $t\le \gamma\lambda$, the minimum of $\MCP$ is attained at $0$, which equals to $0$. Therefore, $\MCP\ge 0$.

Obviously, $\NC=0$ if and only if $t=0$.

\item Suppose we have $\gamma_2>\gamma_1>1$.

For SCAD, increasing $\gamma$ has no influence for $0\le t\le\lambda$. If $\lambda<t<\gamma_1\lambda$, then $\lambda<t<\gamma_2\lambda$. Note that
\[\frac{\mathrm{d}}{\mathrm{d}\gamma}\frac{\gamma\lambda t-0.5(t^2+\lambda^2)}{\gamma-1}=\frac{(t-\lambda)^2}{2(\gamma-1)^2}\ge 0.\]
Therefore,  $\SCADtwo\ge \SCADone$.  If $t>\gamma_1\lambda$, then we have two cases: $t>\gamma_2\lambda$ or $\lambda<t\le \gamma_2\lambda$.  If $t>\gamma_2\lambda$, then $\SCADtwo=\frac{\gamma_2+1}{2}\lambda^2>\frac{\gamma_1+1}{2}\lambda^2=\SCADone$. If $\lambda<t\le \gamma_2\lambda$, then
\[\SCADtwo-\SCADone = \frac{\gamma_2\lambda t -0.5(t^2+\lambda^2)}{\gamma_2-1}-\frac{\gamma_1+1}{2}\lambda^2=\frac{-t^2+2\gamma_2\lambda t-\lambda^2-(\gamma_1+1)(\gamma_2-1)\lambda^2}{2(\gamma_2-1)}.\]
Note that in this case $\gamma_1\lambda\le t\le \gamma_2\lambda$, the above equation attains its minimum at $\gamma_1\lambda$, thus 
\[\SCADtwo-\SCADone \ge\frac{(\gamma_2-\gamma_1)(\gamma_1-1)\lambda^2}{2(\gamma_2-1)}>0.\]
The last inequality follows from the condition $\gamma_2>\gamma_1>1$. Therefore, we conclude that $\SCAD$ is increasing with respect to $\gamma$.

For MCP, if $0\le t<\gamma_1\lambda$, then $0\le t<\gamma_2\lambda$, thus $\MCPtwo=\lambda t-\frac{t^2}{2\gamma_2}>\lambda t-\frac{t^2}{2\gamma_1}=\MCPone$. If $t\ge \gamma_1\lambda$, we have two cases: $t\ge \gamma_2\lambda$ or $0<t<\gamma_2\lambda$. If $t\ge \gamma_2\lambda$, then $\MCPtwo=\frac{\gamma_2\lambda^2}{2}>\frac{\gamma_1\lambda^2}{2}=\MCPone$. If $0\le t\le \gamma_2\lambda$, then
\[\MCPtwo-\MCPone = \lambda t-\frac{t^2}{2\gamma_2}-\frac{\gamma_1\lambda^2}{2}>\frac{(\gamma_2-\gamma_1)\gamma_1\lambda^2}{2\gamma_2}>0.\]
Therefore, we conclude that $\MCP$ is increasing with respect to $\gamma$.

\item Note that as $\gamma\rightarrow \infty$,
\[\frac{\gamma\lambda |t|-0.5(t^2+\lambda^2)}{\gamma-1}\rightarrow \lambda|t|\quad,\quad\lambda|t|-\frac{t^2}{2\gamma}\rightarrow\lambda |t|.\]
The result follows easily.

\item Consider the second order derivative of $\NC$.

For SCAD,
\[\frac{\mathrm{d}^2}{\mathrm{d}^2t}\SCAD=
\begin{cases}
0,\quad \text{if } 0\le t\le \lambda \text{ or } t>\gamma\lambda, \\
-\frac{1}{\gamma-1},\quad\text{if } t\le \gamma\lambda.
\end{cases}
\] 

For MCP,
\[\frac{\mathrm{d}^2}{\mathrm{d}^2t}\MCP=
\begin{cases}
-\frac{1}{\gamma},\quad \text{if }0\le t\le \gamma\lambda,\\
0,\quad\text{if } t\ge \gamma\lambda.
\end{cases}
\]

Therefore, $\NC$ is concave over $[0,\infty)$ since the second order derivative is non-positive. Besides, by the definition of concave function we know that $\varphi_{\lambda,\gamma}(t)\le \varphi_{\lambda,\gamma}(s)+\varphi_{\lambda,\gamma}'(s)(t-s)$ for $s,t\ge 0$.
\end{enumerate}
\end{proof}

\section*{A Novel Tensor Sparsity Measure}
Recall that the novel tensor sparsity measure is defined as
\begin{equation}
\Phi_{\lambda,\gamma}(\calA) = \sum_{i=1}^{n_1}\sum_{j=1}^{n_2}\sum_{k=1}^{n_3}\varphi_{\lambda,\gamma}(\calA_{ijk}).
\end{equation}
Here, we may set $\varphi_{\lambda,\gamma}$ to be $\varphi^{\mathrm{SCAD}}_{\lambda,\gamma}$ or $\varphi^{\mathrm{MCP}}_{\lambda,\gamma}$. We have the following proposition:
\begin{prop} For $\calA\in\bbR^{n_1\times n_2\times n_3}$, $\Phi_{\lambda,\gamma}(\calA)$ satisfies:
\begin{enumerate}[(i)]
\itemsep0em 
\item{} $\Phi_{\lambda,\gamma}(\calA)\ge 0$ with the equality holds iff $\calA=0$;
\item{} $\Phi_{\lambda,\gamma}(\calA)$ is concave with respect to $|\calA|$;
\item{} $\Phi_{\lambda,\gamma}(\calA)$ is increasing in $\gamma$, $\Phi_{\lambda,\gamma}(\calA)\le\lambda \|\calA\|_1$ and $\lim_{\gamma\rightarrow \infty}\Phi_{\lambda,\gamma}(\calA)=\lambda\|\calA\|_1$.
\end{enumerate}
\end{prop}

\begin{proof} Note that $\Phi_{\lambda,\gamma}(\calA)$ is separable with respect to each entry of $\calA$. Thus, applying the related properties of $\NC$ to each entry $\calA_{ijk}$, we immediately get the result.
\end{proof}

\section*{A Novel Tensor Rank Penalty}
Suppose $\calA$ has t-SVD $\calA = \calU*\calS*\calV^*$, we define the $\gamma-$norm of $\calA$ as 
\begin{equation}
\|\calA\|_{\gamma} =\frac{1}{n_3}\sum_{i,k} \varphi_{1,\gamma}(\overline{\calS}(i,i,k)).
\end{equation}
The tensor $\gamma-$norm enjoys the following properties.
\begin{prop} For $\calA\in\bbR^{n_1\times n_2\times n_3}$, suppose $\calA$ has t-SVD $\calA = \calU*\calS*\calV^*$, then $\|\calA\|_{\gamma}$ satisfies:
\begin{enumerate}[(i)]
\itemsep0em
\item{} $\|\calA\|_{\gamma}\ge 0$ with equality holds iff $\calA=0$;
\item{} $\|\calA\|_{\gamma}$ is increasing in $\gamma$, $\|\calA\|_{\gamma}\le \|\calA\|_*$ and $\lim_{\gamma\rightarrow\infty}\|\calA\|_{\gamma}=\|\calA\|_*$;
\item{} $\|\calA\|_{\gamma}$ is concave with respect to $\{\overline{\calS}(i,i,k)\}_{i,k}$;
\item{} $\|\calA\|_{\gamma}$ is orthogonal invariant, \ie, for any orthogonal tensors $\mathcal{P}\in\bbR^{n_1\times n_1\times n_3},\calQ\in\bbR^{n_2\times n_2\times n_3}$, $\|\mathcal{P}*\calA*\calQ\|_{\gamma}=\|\calA\|_{\gamma}$.
\end{enumerate}
\end{prop}

\begin{proof}
\begin{enumerate}[(i)]
\item Since $\NCg\ge0$ and $\|\calA\|_{\gamma}$ is the sum of $\NCg$, we immediately know $\|\calA\|_{\gamma}\ge 0$. If  $\calA=0$, then obviously $\calS=0$ and $\overline{\calS}=0$, which implies $\|\calA\|_{\gamma}=0$.  On the other hand, $\|\calA\|_{\gamma}=0$ implies $\overline{\calS}(i,i,k)=0$. However, $\overline{\calS}$ is f-diagonal, thus $\overline{\calS}=0$, and $\calA=\calU*\ifft(\overline{\calS},3)*\calV^*=0$.

\item Since $\NCg$ is increasing with respect to $\gamma$ and $\|\calA\|_{\gamma}$ is the sum of $\NCg$, using the properties of $\NC$ we know $\|\calA\|_{\gamma}$ is increasing with respect to $\gamma$ and 
\[\lim_{\gamma\rightarrow\infty}\|\calA\|_{\gamma}=\lim_{\gamma\rightarrow\infty}\frac{1}{n_3}\sum_{i,k}\varphi_{1,\gamma}(\overline{\calS}(i,i,k))=\frac{1}{n_3}\sum_{i,k}|\overline{\calS}(i,i,k)|=\|\calA\|_*.\]
Combining the above facts, we also get $\|\calA\|_{\gamma}\le\|\calA\|_*.$

\item This follows from the fact that $\NC$ is concave over $[0,\infty).$

\item Since $\calA$ has t-SVD $\calA = \calU*\calS*\calV^*$, we claim that $\calP*\calU*\calS*\calV^**\calQ^*$ is the t-SVD of $\calP*\calA*\calQ^*$. $\calS$ is already f-diagonal, so we only need to verify $\calP*\calU$ and $(\calV^**\calQ^*)^*=\calQ*\calV$ are orthogonal.
\[(\calP*\calU)*(\calP*\calU)^*=\calP*\calU*\calU^**\calP=\calI.\]
Other equalities are similar to verify. Therefore, $\|\calP*\calA*\calQ^*\|_{\gamma}=\frac{1}{n_3}\sum_{i,k} \varphi_{1,\gamma}(\overline{\calS}(i,i,k))=\|\calA\|_{\gamma}$.
\end{enumerate}
\end{proof}

\section*{Generalized Thresholding Operators}
\begin{thm}\label{upperboundA}
We can view $\Phi_{\lambda,\mu}(\calX)$ as a function of $|\calX|$, and $\|\calX\|_{\gamma}$ as a function of $\{\overline{\calS}(i,i,k)\}_{i,k}$. For any $\calX^{\textnormal{old}}$, let
\[
Q_{\lambda,\gamma}(\calX|\calX^{\textnormal{old}})=\Phi_{\lambda,\gamma}(\calX^{\textnormal{old}})+\sum_{i,j,k}\varphi'_{\lambda,\gamma}(|\calX^{\textnormal{old}}_{ijk}|)(|\calX_{ijk}|-|\calX^{old}_{ijk}|),\]
\[
Q_{\gamma}(\calX|\calX^{\textnormal{old}})= \|\calX^{\textnormal{old}}\|_{\gamma}+\frac{1}{n_3}\sum_{i,k}\varphi_{1,\gamma}'(\overline{\calS}^{\old}_{iik})(\overline{\calS}_{iik}-\overline{\calS}^{\textnormal{old}}_{iik}),
\]
then
\begin{equation}
\begin{aligned}
Q_{\lambda,\gamma}(\calX^{\mathrm{old}}|\calX^{\mathrm{old}})=\Phi_{\lambda,\gamma}(\calX^{\mathrm{old}})\qquad,\qquad\Phi_{\lambda,\gamma}(\calX)\le Q_{\lambda,\gamma}(\calX|\calX^{\textnormal{old}}),\\
Q_{\gamma}(\calX^{\mathrm{old}}|\calX^{\mathrm{old}})=\|\calX^{\mathrm{old}}\|_{\gamma}\qquad,\qquad\|\calX\|_{\gamma}\le Q_{\gamma}(\calX|\calX^{\textnormal{old}}).
\end{aligned}
\end{equation}
\end{thm}

\begin{proof}
Recall that for any $s,t\ge 0$ we have $\NC\le \varphi_{\lambda,\gamma}(s)+\varphi_{\lambda,\gamma}(s)(t-s)$.
\[
\Phi_{\lambda,\gamma}(\calX)=\sum_{i,j,k}\varphi_{\lambda,\gamma}(\calX_{ijk})=\sum_{i,j,k}\varphi_{\lambda,\gamma}(|\calX_{ijk}|)\le\sum_{i,j,k}\left(\varphi_{\lambda,\gamma}(|\calX^{\old}_{ijk}|)+\varphi_{\lambda,\gamma}'(|\calX^{\old}_{ijk}|)(|\calX_{ijk}|-|\calX^{\old}_{ijk}|)\right)=Q_{\lambda,\gamma}(\calX|\calX^{\old}).
\]
\[
\|\calX\|_{\gamma}=\frac{1}{n_3}\sum_{i,k}\varphi_{1,\gamma}(\overline{\calS}_{ijk})\le
\frac{1}{n_3}\sum_{i,k}\left(\varphi_{1,\gamma}(\overline{\calS}^{\old}_{ijk})+\varphi_{1,\gamma}'(\overline{\calS}^{\old}_{ijk})(\overline{\calS}_{ijk}-\overline{\calS}_{ijk}^{\old})\right)=Q_{\gamma}(\calX|\calX^{\old}).
\]
\end{proof}

In the following soft thresholding operator is defined as $\calT_{\lambda}(z)=\textnormal{sgn}(z)[|z|-\lambda]_+$, which is the proximal operator of $\ell_1-$norm.

\begin{defn}[\bf{Generalized soft thresholding}]
Suppose $\calX,\calW\in\bbR^{n_1\times n_2\times n_3}$, the generalized soft thresholding operator is defined as
\begin{equation}
[\calT_{\calW}(\calX)]_{ijk} = \calT_{\calW_{ijk}}(\calX_{ijk}).
\end{equation}
\end{defn}

\begin{thm}\label{thm:generalizedthresholdingA}
For $\forall\mu>0$, let $\calW_{ijk}=\varphi_{\lambda,\gamma}'(|\calX^{\old}_{ijk}|)/\mu$, then 
\begin{equation}
\calT_{\calW}(\calY)=\arg\min_{\calX}Q_{\lambda,\gamma}(\calX|\calX^{\old})+\frac{\mu}{2}\|\calX-\calY\|_F^2.
\end{equation}
\end{thm}

\begin{proof}
In fact, 
\[
\begin{aligned}
\arg\min_{\calX} Q_{\lambda,\gamma}(\calX|\calX^{\old})+\frac{\mu}{2}\|\calX-\calY\|_F^2 =&\arg\min_{\calX} \sum_{ijk}\varphi_{\lambda,\gamma}'(|\calX^{\old}_{ijk}|)|\calX_{ijk}|+\frac{\mu}{2}\|\calX-\calY\|_F^2\\
=&\arg\min_{\calX} \sum_{i,j,k}\left(\varphi_{\lambda,\gamma}'(|\calX^{\old}_{ijk}|)|\calX_{ijk}|+\frac{\mu}{2}(\calX_{ijk}-\calY_{ijk})^2\right),
\end{aligned}
\]
which is separable to each entries of $\calX$. Consider $\calX_{ijk}$, according to the property of soft thresholding operator, the minimum of $\varphi_{\lambda,\gamma}'(|\calX^{\old}|)|\calX_{ijk}|+\frac{\mu}{2}(\calX_{ijk}-\calY_{ijk})^2$ is attained at $\calT_{\varphi_{\lambda,\gamma}'(|\calX^{\old}_{ijk}|)/\mu}(\calY_{ijk})$. Therefore, let $\calW_{ijk}=\varphi_{\lambda,\gamma}'(|\calX^{\old}_{ijk}|)/\mu$, by the definition of generalized soft thresholding, we immediately get
\[\calT_{\calW}(\calY)=\arg\min_{\calX}Q_{\lambda,\gamma}(\calX|\calX^{\old})+\frac{\mu}{2}\|\calX-\calY\|_F^2.\]
\end{proof}

\begin{defn}[\bf{Generalized t-SVT}]Suppose a 3-way tensor $\calY$ has t-SVD $\calY=\calU*\calS*\calV^*$, $\calW$ is a tensor with the same shape of $\calY$, the generalized tensor singular value thresholding operator is defined as 
\begin{equation}
\calD_{\calW}(\calY)=\calU*\tilde{\calS}*\calV^*,
\end{equation}
where $\tilde{\calS}=\textnormal{\ifft}(\calT_{\calW}(\overline{\calS}),3)$.
\end{defn}

\begin{thm}\label{thm:generalizedtSVTA}
For $\forall \mu>0$, let $\calW_{ijk}=\delta_{i}^j\varphi_{1,\gamma}'(\overline{\calS}^{\old}_{iik})/\mu$ where $\delta_i^j$ is the Kronecker symbol, then
\begin{equation}
\calD_{\calW}(\calY)=\arg\min_{\calX}Q_{\gamma}(\calX|\calX^{\old})+\frac{\mu}{2}\|\calX-\calY\|_F^2.
\end{equation}
\end{thm}

\begin{proof}
In fact,
\[\begin{aligned}
\arg\min_{\calX} Q_{\gamma}(\calX|\calX^{\old})+\frac{\mu}{2}\|\calX-\calY\|_F^2
=&\arg\min_{\calX}\frac{1}{n_3}\sum_{i,k}\varphi_{1,\gamma}'(\overline{\calS}^{\old}_{iik})\overline{\calS}_{iik}+\frac{\mu}{2}\|\calX-\calY\|_F^2\\
=&\arg\min_{\calX}\frac{1}{n_3}\sum_{i,k}\varphi_{1,\gamma}'(\overline{\calS}^{\old}_{iik})\overline{\calS}_{iik}+\frac{\mu}{2n_3}\sum_{k}\|\overline{X}^{(k)}-\overline{Y}^{(k)}\|_F^2\\
=&\arg\min_{\calX}\frac{1}{n_3}\sum_{k}\left(\sum_{i}\varphi_{1,\gamma}'(\overline{\calS}^{\old}_{iik})\overline{S}^{(k)}_{i,i}+\frac{\mu}{2}\|\overline{X}^{(k)}-\overline{Y}^{(k)}\|_F^2\right).
\end{aligned}\]
This optimization problem is separable in transformation domain with respect to each slice. Consider the subproblem of minimizing $\sum_{k}\varphi_{1,\gamma}'(\overline{\calS}^{\old}_{iik})\overline{S}^{(k)}_{i,i}+\frac{\mu}{2}\|\overline{X}^{(k)}-\overline{Y}^{(k)}\|_F^2$, suppose $\calY$ has t-SVD $\calY=\calU*\calR*\calV$ or equivalently $\overline{Y}=\overline{U}\,\overline{R}\,\overline{T}^*$, we have
\[\begin{aligned}
\sum_{k}\varphi_{1,\gamma}'(\overline{\calS}^{\old}_{iik})\overline{S}^{(k)}_{i,i}+\frac{\mu}{2}\|\overline{X}^{(k)}-\overline{Y}^{(k)}\|_F^2 
=&\sum_{k}\varphi_{1,\gamma}'(\overline{\calS}^{\old}_{iik})\overline{S}^{(k)}_{i,i}+\frac{\mu}{2}\|\overline{X}^{(k)}-\overline{U}^{(k)}\,\overline{R}^{(k)}\,(\overline{V}^{(k)})^*\|_F^2\\
=&\sum_{k}\varphi_{1,\gamma}'(\overline{\calS}^{\old}_{iik})\overline{S}^{(k)}_{i,i}+\frac{\mu}{2}\|(\overline{U}^{(k)})^*\overline{X}^{(k)}\overline{V}^{(k)}-\,\overline{R}^{(k)}\,\|_F^2 .
\end{aligned}\]
However, note that $\overline{S}^{(k)}_{i,i}$ is still the singular values of $(\overline{U}^{(k)})^*\overline{X}^{(k)}\overline{V}^{(k)}$, simple calculation reveals that we obtain the minimum if $(\overline{U}^{(i)})^*\overline{X}^{(k)}\overline{V}^{(k)}$ is diagonal with $i$-th diagonal element equals to $\calT_{\varphi_{1,\gamma}'(\overline{S^{\old}_{iik}})/\mu}(\overline{R}^{k}_{i,i})$. That is, let $W$ be a matrix with $W_{i,j}^{(k)}=\delta_i^j\varphi_{1,\gamma}'(\overline{S^{\old}_{iik}})/\mu$, then $\overline{X}=\overline{U}\calT_W{\overline{R}}\,\overline{V}$. Transform back to original domain, by the definition of generalized t-SVT, we get $\calX=\calU*\tilde{\calS}*\calV^*$ where $\tilde{\calS}=\ifft(\calT_{\calW}(\overline{\calR}),3)$ and $\calW_{ijk}=\delta_{i}^j\varphi_{1,\gamma}'(\overline{\calS}^{\old}_{iik})/\mu$.
\end{proof}

\section*{ADMM for Proposed Non-convex Tensor Recovery Models and Algorithms}
\begin{lem}\label{lemma}
	Suppose $\phi:I\subset\mathbb{R}\rightarrow \mathbb{R}$ is a monotone  non-decreasing differentiable concave function and $x_*\in I$ satisfies $x_*\in\arg\min_{x\in I}\phi(x_*)+\phi'(x_*)(x-x_*)$, then $x_*$ is a minimal point of $\phi(x)$ on $I$. 
\end{lem}
\begin{proof}
	Note that $\phi(x_*)\le \phi(x_*)+\phi'(x_*)(x-x_*)$ for $x\in I$ indicates $\phi'(x_*)(x-x_*)\ge 0,\forall x\in I$. Since $\phi$ is monotone non-decreasing, $(x-x_*)(\phi(x)-\phi(x_*))\ge 0$. Multiplying these two inequalities, we obtain $\phi'(x_*)(\phi(x)-\phi(x_*))(x-x_*)^2\ge 0$. Therefore, for any $x\ne x_*$, $\phi'(x_*)(\phi(x)-\phi(x_*))\ge 0$. If $\phi'(x_*)>0$, we  immediately have $\phi(x)\ge\phi(x_*)$. If $\phi'(x_*)=0$, then obviously $x_*$ is a minimizer. 
\end{proof}

\noindent{\bf Remark}: The SCAD and MCP functions are monotone non-decreasing on $[0,\infty)$, thus the lemma can be applied to them. Furthermore, $\Phi_{\lambda,\gamma}(\calX)$ is separable regarded as a function of $\calX$, and $\|\calX\|_{\gamma}$ is separable regarded as a function of $\{\overline{S}(i,i,k)\}_{i,k}$. Therefore, the lemma can also be applied to $\Phi_{\lambda,\gamma}(\calX)$ and $\|\calX\|_{\gamma}$. 
\subsection{ADMM for Non-convex Tensor Completion }
Consider problem
\begin{equation}
\min_{\calX\in D} Q_{\gamma}(\calM|\calX^{\old})\quad\text{s.t. } \calM=\calX.
\end{equation}
The augmented Lagrangian function is given by
\begin{equation}
L(\calM,\calX,\calY)=Q_{\gamma}(\calM|\calX^{\old})+\langle\calM-\calX,\calY\rangle+\frac{\mu}{2}\|\calM-\calX\|_F^2.
\end{equation}
According to the ADMM algorithm, we have the following iteration scheme:
\begin{gather}
\calM_{l+1} = \arg\min_{\calM} Q_{\gamma}(\calM|\calX^{\old})+\frac{\mu}{2}\|\calM-(\calX_l-\frac{1}{\mu}\calY_l)\|_F^2,\notag\\
\calX_{l+1} = \arg\min_{\calX\in D} \|\calX-(\calM_{l+1}+\frac{1}{\mu}\calY_l)\|_F^2,\\
\calY_{l+1} = \calY_k+\mu(\calM_{l+1}-\calX_{l+1}).\notag
\end{gather}
The sub-problem of updating $\calM_{l+1}$ can be solved by generalized t-SVT as indicated in Theorem \ref{thm:generalizedtSVTA}. The sub-problem of updating $\calX_{l+1}$ has a closed-form solution: $\calX_{l+1} =(\calM_{l+1}+\frac{1}{\mu}\calY)\circledast(1-\Omega)+\calO\circledast \Omega$, where $\circledast$ is element-wise product. 

\begin{thm}
	The iteration sequence generated by $\calX^{t+1}\in \arg\min_{\calO_{\Omega}=\calX_{\Omega}}Q_{\gamma}(\calX|\calX^t)$ is non-increasing, \ie, $\|\calX^{t+1}\|_{\gamma}\le \|\calX^t\|_{\gamma}$ and converges to some $Q^*$. Besides, there exists a subsequence $\{\calX^{i_k}\}_{k=1}^{\infty}$ converges to a minimal point $\calX_*$ of $\|\calX\|_{\gamma}$ on $\{\calX|\calO_{\Omega}=\calX_{\Omega}\}$.
\end{thm}
\begin{proof}
		Since $Q_{\gamma}(\calX^t|\calX^t)=\|\calX^t\|_{\gamma}$ and $\calX^{t+1}\in\arg\min_{\calO_{\Omega}=\calX_{\Omega}}Q_{\gamma}(\calX|\calX^t)$, we immediately have $Q_{\gamma}(\calX^{t+1}|\calX^t)\le \|\calX^t\|_{\gamma}$. On the other hand, $\|\calX^{t+1}\|_{\gamma}\le Q_{\gamma}(\calX^{t+1}|\calX^t)$, thus $\|\calX^{t+1}\|_{\gamma}\le \|\calX^t\|_{\gamma}$.
		
		 The sequence $\{\|\calX^t\|_{\gamma}\}_{l=0}^{\infty}$ is bounded below by $0$ and non-increasing, therefore converges to some limit point $Q^*\ge 0$.
		 
		  Besides, $\{\calX^t\}_{t=0}^{\infty}$ lie in the compact set $\{\calX|\|\calX\|_{\gamma}\le\|\calX^0\|_{\gamma}\}$, so there exists a subsequence $\{\calX^{i_k}\}_{k=1}^{\infty}$ that is convergent. Without loss of generality, we assume $\calX^{i_k}\rightarrow \calX_*$, then we must have $\calX_*\in\arg\min_{\calO_{\Omega}=\calX_{\Omega}}Q_{\gamma}(\calX|\calX_*)$. Then by Lemma \ref{lemma}, we conclude that $\calX_*$ is a minimizer of $\|\calX\|_{\gamma}$ on $\{\calX|\calO_{\Omega}=\calX_{\Omega}\}$.
\end{proof}

\subsection{ADMM for Non-convex Tensor Robust PCA}
Consider problem
\begin{equation}
\min_{\calL,\calE}Q_{\gamma_1}(\calL|\calL^{\old})+Q_{\lambda,\gamma_2}(\calE|\calE^{\old})\quad\text{s.t. }\calL+\calE=\calX.
\end{equation}
The augmented Lagrangian function is 
\begin{equation}
L(\calL,\calE,\calY) =Q_{\gamma_1}(\calL|\calL^{\old})+Q_{\lambda,\gamma_2}(\calE|\calE^{\old}) +\langle\calY,\calL+\calE-\calX\rangle+\frac{\mu}{2}\|\calL+\calE-\calX\|_F^2.
\end{equation}
According to the ADMM algorithm, we may iterate variables as following:
\begin{gather}
\calL_{l+1}\!=\!\arg\min_{\calL}Q_{\gamma_1}(\calL|\calL^{\old})+\frac{\mu}{2}\|\calL \!-\!(\calX \!-\!(\calE_l\!+\!\frac{1}{\mu}\calY_l))\|_F^2,\notag\\
\calE_{l+1}\!=\!\arg\min_{\calE}Q_{\lambda,\gamma_2}(\calE|\calE^{\old})\!+\!\frac{\mu}{2}\|\calE\!-\!(\calX\!-\!(\calL_{l+1}\!+\!\frac{1}{\mu}\calY_l))\|_F^2,\notag\\
\calY_{l+1}=\calY_l+\mu(\calL_{l+1}+\calE_{l+1}-\calX).
\end{gather}
The sub-problem of updating $\calL$ and $\calE$ has closed-form solutions using Theorem \ref{thm:generalizedthresholdingA} and Theorem \ref{thm:generalizedtSVTA}. 

\begin{thm}
	The iteration sequence generated by 
	\[(\calL^{t+1},\calE^{t+1})\in\arg\min_{\calL+\calE=\calX} Q_{\gamma_1}(\calL|\calL^t)+Q_{\lambda,\gamma_2}(\calE|\calE^t)\]
	is non-increasing, \ie, $\|\calL^{t+1}\|_{\gamma_1}+\Phi_{\lambda,\gamma_2}(\calE^{t+1}) \le \|\calL^{t}\|_{\gamma_1}+\Phi_{\lambda,\gamma_2}(\calE^{t})$ and converges to some $Q^*$. Besides, there exists a subsequence $\{(\calL^{i_k},\calE^{i_k})\}_{k=1}^{\infty}$ converges to a minimal point $(\calL_*,\calE_*)$ of $\|\calL\|_{\gamma_1}+\Phi_{\lambda,\gamma_2}(\calE)$ on $\{(\calL,\calE)|\calL+\calE=\calX\}$.
\end{thm}
\begin{proof}
	Since $Q_{\gamma_1}(\calL^t|\calL^t)=\|\calL^t\|_{\gamma},Q_{\lambda,\gamma_2}(\calE^t|\calE^t)=\Phi_{\lambda,\gamma_2}(\calE_l)$, and $(\calL^{t+1},\calE^{t+1})\in\arg\min_{\calL+\calE=\calX} Q_{\gamma_1}(\calL|\calL^t)+Q_{\lambda,\gamma_2}(\calE|\calE^t)$, we immediately have $Q_{\gamma_1}(\calL^{t+1}|\calL^t)+Q_{\lambda,\gamma_2}(\calE^{t+1}|\calE^t)\le \|\calL^t\|_{\gamma_1}+\le \Phi_{\lambda,\gamma_2}(\calE^t)$. On the other hand, $\|\calL^{t+1}\|_{\gamma_1}\le Q_{\gamma_1}(\calL^{t+1}|\calL^t)$ and $\Phi_{\lambda,\gamma_2}(\calE^{t+1}) \le Q_{\lambda,\gamma_2}(\calE^{t+1}|\calE^t)$, thus $\|\calL^{t+1}\|_{\gamma_1}+\Phi_{\lambda,\gamma_2}(\calE^{t+1})\le \|\calL^t\|_{\gamma_1}+\Phi_{\lambda,\gamma_2}(\calE^t)$.
	 
	 The sequence $\{\|\calL^t\|_{\gamma_1}+\Phi_{\lambda,\gamma_2}(\calE^t)\}_{t=0}^{\infty}$ is bounded below by $0$ and non-increasing, therefore converges to some limit point $Q^*\ge 0$. 
	 
	 Besides, $\{(\calL^t,\calE^t)\}_{t=0}^{\infty}$ lie in the compact set $\{(\calL,\calE)|\|\calL\|_{\gamma_1}+\Phi_{\lambda,\gamma_2}(\calE)\le\|\calL^0\|_{\gamma_1}+\Phi_{\lambda,\gamma_2}(\calE^0)\}$, so there exists a subsequence $\{(\calL^{i_k},\calE^{i_k})\}_{k=1}^{\infty}$ that is convergent. Without loss of generality, we assume $(\calL^{i_k},\calE^{i_k})\rightarrow (\calL_*,\calE_*)$, then we must have $(\calL_*,\calE_*)\in\arg\min_{\calL+\calE=\calX}Q_{\gamma_1}(\calL|\calL_*)+Q_{\lambda,\gamma_2}(\calE|\calE_*)$. Then by Lemma \ref{lemma}, we conclude that $(\calL_*,\calE_*)$ is a minimizer of $\|\calL\|_{\gamma_1}+\Phi_{\lambda,\gamma_2}(\calE)$ on $\{(\calL,\calE)|\calL+\calE=\calX\}$.
\end{proof}

\end{document}